\makeatletter
\def\cup@reference@code{%
  \RequirePackage[style=numeric,sorting=none,backend=biber,natbib]{biblatex}%
  \renewcommand*{\bibfont}{\footnotesize}%
}
\makeatother

\documentclass[
  journal=medium,
  manuscript=article,
  manuscriptlabel={Research Preprint},
  logo=false,
  year=2020,
  volume=37,
]{cup-journal}

\newif\ifcambridge
\cambridgetrue

\usepackage[utf8]{inputenc}
\usepackage[T1]{fontenc}
\usepackage{amsmath,amssymb,amsthm}
\usepackage{graphicx}
\usepackage{booktabs}
\usepackage{algorithm}
\usepackage{algpseudocode}
\usepackage[colorlinks=true,linkcolor=blue,citecolor=blue]{hyperref}
\usepackage{xcolor}
\usepackage{bm}
\usepackage{mathtools}
\usepackage{multirow}
\usepackage{enumitem}
\usepackage{float}
\setlist{nosep}

\newenvironment{widefig}[1][t]{\begin{figure}[#1]}{\end{figure}}

\newtheorem{theorem}{Theorem}[section]
\newtheorem{lemma}[theorem]{Lemma}
\newtheorem{proposition}[theorem]{Proposition}
\newtheorem{corollary}[theorem]{Corollary}
\theoremstyle{definition}
\newtheorem{definition}[theorem]{Definition}
\newtheorem{example}[theorem]{Example}
\theoremstyle{remark}
\newtheorem{remark}[theorem]{Remark}

\newcommand{\tr}{\operatorname{tr}}
\newcommand{\diag}{\operatorname{diag}}
\newcommand{\GM}{\operatorname{GM}}
\newcommand{\AM}{\operatorname{AM}}
\newcommand{\E}{\mathbb{E}}
\newcommand{\Var}{\operatorname{Var}}

\newcommand{\R}{\mathbb{R}}

\DeclareMathOperator*{\argmin}{arg\,min}

\title{What Trace Powers Reveal About Log-Determinants:\\ Closed-Form Estimators, Certificates, and Failure Modes}
\author{Piyush Sao}
\affiliation{Oak Ridge National Laboratory, Oak Ridge, TN 37831, USA}
\email{saopk@ornl.gov}

\makeatletter
\@ifpackageloaded{xpatch}{}{\usepackage{xpatch}}
\def\cup@journal@name{What Trace Powers Reveal About Log-Determinants?}
\def\cup@manuscript{preprint}
\patchcmd{\@maketitle}{\vspace*{\baselineskip}}{\vspace*{0.2\baselineskip}}{}{}
\patchcmd{\@maketitle}{(\cup@year), {\volumefont\cup@vol}, \thepage--\pageref{LastPage}}{}{}{}
\makeatother

\makeatletter
\renewcommand*\cup@maketitle@extras@hook{%
  \begingroup
  \renewcommand{\thefootnote}{\fnsymbol{footnote}}%
  \footnotetext{This manuscript has been authored by UT-Battelle, LLC under Contract No. DE-AC05-00OR22725 with the U.S. Department of Energy. The publisher, by accepting the article for publication, acknowledges that the United States Government retains a non-exclusive, paid-up, irrevocable, world-wide license to publish or reproduce the published form of this manuscript, or allow others to do so, for United States Government purposes. The Department of Energy will provide public access to these results of federally sponsored research in accordance with the DOE Public Access Plan (\url{http://energy.gov/downloads/doe-public-access-plan}).}%
  \endgroup
}
\makeatother

\begin{document}

\begin{abstract}

Computing $\log\det(A)$ for large symmetric positive definite matrices arises in Gaussian process inference and Bayesian model comparison. Standard scalable methods use many matrix--vector products and polynomial approximations to $\tr(\log A)$. We study a different information model: access to a few trace powers $p_k = \tr(A^k)$, natural when matrix powers are computed or repeated squaring is feasible.

Classical moment-based approximations Taylor-expand $\log(\lambda)$ around the arithmetic mean. This requires $|\lambda - \AM| < \AM$ and diverges when $\kappa > 4$. We work instead with the moment-generating function $M(t) = \E[X^t]$ for normalized eigenvalues $X = \lambda/\AM$. Since $M'(0) = \E[\log X]$, the log-determinant becomes $\log\det(A) = n(\log \AM + M'(0))$---the problem reduces to estimating a derivative at $t = 0$. Trace powers give $M(k)$ at positive integers, but interpolating $M(t)$ directly is ill-conditioned due to exponential growth. The transform $K(t) = \log M(t)$ compresses this range. Normalization by $\AM$ ensures $K(0) = K(1) = 0$. With these anchors fixed, we interpolate $K$ through $m+1$ consecutive integers and differentiate to estimate $K'(0)$. However, this local interpolation cannot capture arbitrary spectral features.

We prove a fundamental limit: no continuous estimator using finitely many positive moments can be uniformly accurate over unbounded conditioning. Positive moments downweight the spectral tail; $K'(0) = \E[\log X]$ is tail-sensitive. This motivates guaranteed bounds. From the same traces we derive upper bounds on $(\det A)^{1/n}$. Given a spectral floor $r \leq \lambda_{\min}$, we obtain moment-constrained lower bounds, yielding a provable interval for $\log\det(A)$. A gap diagnostic indicates when to trust the point estimate and when to report bounds. All estimators and bounds cost $O(m)$, independent of $n$. For $m \in \{4, \ldots, 8\}$, this is effectively constant time.

\end{abstract}


\section{Introduction}
\label{sec:intro}


Estimating $\log\det(A)$ for large symmetric positive definite (SPD) matrices is central to Gaussian process inference~\citep{rasmussen2006}, Bayesian model comparison~\citep{mackay1992,kass1995}, and uncertainty quantification~\citep{kennedy2001,stuart2010}. Standard methods combine stochastic trace estimation~\citep{hutchinson1990,avron2011} with polynomial~\citep{han2015} or Lanczos-based~\citep{ubaru2017,saibaba2017} approximations of $\tr(\log A)$. All require repeated matrix-vector products.

We study a different information model with access to a small number of \emph{trace powers}.
Given only
\begin{equation}
p_k = \tr(A^k), \qquad k=1,2,\dots,m,
\end{equation}
or unbiased estimates thereof, what can we infer about $\log\det(A)$? Such data arise naturally in two settings. First, workflows that form matrix powers can obtain additional traces cheaply: given $A^p$ and $A^q$, computing $\tr(A^{p+q})$ from the diagonal of $A^p A^q$ costs only $O(n^2)$ time~\citep{higham2008}. Second, low-precision hardware (tensor cores, TPUs) makes repeated squaring $A \to A^2 \to A^4 \to \cdots$ attractive~\citep{haidar2018}. This motivates our question:

\emph{\textbf{What can (and cannot) be inferred about $\log\det(A)$ from finitely many positive trace powers?}}

\begin{widefig}[t]
\centering
\includegraphics[width=\textwidth]{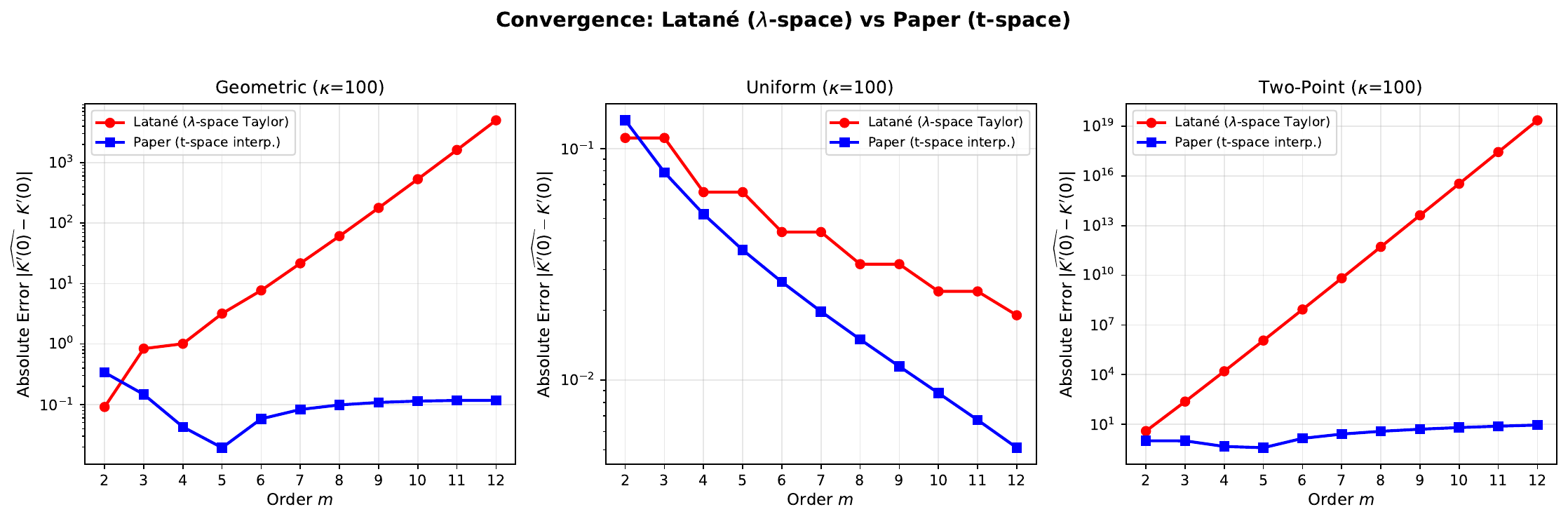}
\caption{Absolute error in estimating $K'(0) = \log(\GM/\AM)$, where $\GM = (\det A)^{1/n}$ is the geometric mean and $\AM = \tr(A)/n$ the arithmetic mean, across three spectrum types ($n=1024$, condition number $\kappa=100$). Latan\'e's $\lambda$-space Taylor expansion (red) diverges exponentially with order. Our exponent-space interpolation (blue) achieves minimum error near $m \in \{4,5,6\}$ and remains stable at higher orders. For the geometric spectrum, our method achieves error $\approx 0.02$ at $m=5$ while Latan\'e's error exceeds $10^3$ at $m=12$.}
\label{fig:latane-comparison}
\end{widefig}


The classical approach expands $\log(\lambda)$ around the arithmetic mean as a Taylor series in central moments---the volatility-drag formula~\citep{latane1959,jean1983}. This converges only when $|\lambda - \AM| < \AM$ for all eigenvalues, a condition violated when $\kappa > 4$. Adding more terms makes the approximation worse, not better (Figure~\ref{fig:latane-comparison}). Why does this natural approach fail so badly?

\subsection{The Information Gap}

The log-determinant $\log\det(A) = \sum_i \log\lambda_i$ depends critically on small eigenvalues: if $\lambda_{\min} \to 0$, the sum diverges to $-\infty$. Trace powers $\tr(A^k) = \sum_i \lambda_i^k$ suppress them: as $\lambda_{\min} \to 0$, its contribution vanishes extremely fast (as $\lambda_{\min}^k \to 0$ for $k \ge 1$). This creates an \emph{information gap}: our data systematically suppress exactly the eigenvalues on which $\log\det(A)$ depends most. The classical volatility-drag formula cannot bridge this gap---it is the Taylor expansion whose $\kappa > 4$ divergence we noted above.

How accurately can we estimate the geometric mean $\GM = (\det A)^{1/n}$ from trace powers? We prove that no finite-trace estimator can uniformly approximate $\GM$ over all SPD spectra; the achievable accuracy depends on the unknown spectrum itself. This impossibility result shifts the goal: \emph{when can we trust an estimator?} Without a universal method, we need to know when estimators succeed, when they fail, and how to detect failure from the data.


\subsection{Our Approach: Exponent-Space Interpolation}

Rather than fight the information gap in $\lambda$-space, we change coordinates and treat $t$ as a continuous exponent. Let $M(t)=\E[X^t]$ for an eigenvalue random variable $X$. With $X=\lambda/\AM$, we have
\begin{equation}
\label{eq:logdet-identity}
\log\det(A) = n\bigl(\log(\AM) + M'(0)\bigr).
\end{equation}
Therefore, given $M(k)$ for $k=0,\dots,m$, we seek to estimate $M'(0)$ from integer samples $M(k)$.
However, directly interpolating $M(t)$ is ill-conditioned: $M(k)$ spans orders of magnitude as $k$ grows, and further we show that in presence of noisy trace powers (floating point and Hutchinson), numerical differentiation amplifies impact of noise on $M'(0)$.

To address the numerical instability, we show that any analytic transform $f(M(t))$ can be used for estimating $M'(0)$ (Theorem~\ref{thm:analytic-transform}); and the choice of $f$ can alleviate the instability and noise sensitivity, and depending on the eigenspectra and noise model, the conditioning of the problem. In particular, under multiplicative noise, $\log$ is the unique variance-stabilizing transform (up to affine scaling), so we work with $K(t)=\log M(t)$. Since $M(0)=1$, we have $K'(0)=M'(0)$.
Trace powers give $K(k)$ at positive integers; the problem becomes: \emph{estimate $K'(0)$ from $K(0), K(1), \ldots, K(m)$}.

\paragraph{Why exponent space works.}
Trace powers give samples at integer $t$, and $K(t)=\log M(t)$ is analytic near $t=0$, so $K'(0)$ is determined by nearby values of $K$. This turns estimation into a numerical differentiation problem in exponent space rather than a divergent Taylor expansion in $\lambda$-space. Differentiating the degree-$m$ Lagrange interpolant at $t=0$ yields simple closed-form weights, producing the distinct family of $k_{0:m}$ estimators (Section~\ref{sec:estimators}).


\subsection{Contributions}

We provide a theoretical and practical framework for
estimating the log-determinant of a
positive-definite matrix using only traces of its
powers. While trace-based estimators can be
computationally attractive, they are not universally
reliable. We therefore ask: under what spectral
conditions can we accurately reconstruct the
log-determinant from a finite number of trace
powers? When will such estimators fail
catastrophically? How can we detect failure from the data? We address these challenges with
six key contributions.

\begin{enumerate}
\item \textbf{How to estimate log-determinant from trace powers?}
We provide a family of estimators (Section~\ref{sec:estimators}) with closed-form weights $L'_j(0) = (-1)^{j-1} \binom{m}{j}/j$ from Lagrange interpolation (Theorem~\ref{thm:weights}).

\item \textbf{When is estimation exact?}
For lognormal spectra, where $K(t) = \mu t + \sigma^2 t^2/2$ is quadratic (Theorem~\ref{thm:lognormal}).

\item \textbf{When must estimation fail?}
No universal estimator exists: moments $\tr(A^j)$ stabilize while $\log\det(A) \to -\infty$ (Theorem~\ref{thm:impossibility}).

\item \textbf{Can we predict failure?}
Yes: Taylor radius $R = \pi/\log\kappa$ determines breakdown (Proposition~\ref{prop:taylorradius}); $m > R$ guarantees divergence.

\item \textbf{Can we certify quality?}
Rigorous bounds $L_k(r) \leq \GM \leq U_k$ from traces alone via $(k+1)$-atom optimization, with complexity independent of matrix dimension (Theorems~\ref{thm:certified-interval},~\ref{thm:compact-support}).

\item \textbf{How does noise affect accuracy?}
Variance grows as $\alpha_m \sim 2^m / m^{5/4}$ (Theorem~\ref{thm:variance}, Proposition~\ref{prop:ell2}); use $m = 4$ for Hutchinson noise.
\end{enumerate}


\paragraph{Preview.}
The remainder of this paper develops these contributions and shows that the information gap, while fundamental, admits practical workarounds for the smooth spectra common in applications. Section~\ref{sec:conclusion} distills deployment recommendations.

\paragraph{Scope.}
We use polynomial interpolation throughout---the natural first step, yielding closed-form weights and well-understood error bounds. The Taylor radius limits and sensitivity to interpolation order $m$ are consequences of this choice; Pad\'e approximants or rational interpolation may extend the convergence region but sacrifice simplicity. Experiments use exact eigenvalues to isolate interpolation bias; Section~\ref{sec:noisy-traces} provides the theory for noisy traces. The bounds are deterministic certificates; probabilistic confidence intervals remain future work.

\paragraph{Organization.}
Section~\ref{sec:related} reviews related work. Sections~\ref{sec:framework}--\ref{sec:exactness} develop the CGF framework and $k_{0:m}$ estimators. Section~\ref{sec:impossibility} establishes fundamental limitations. Section~\ref{sec:bounds} presents guaranteed bounds. Section~\ref{sec:noisy-traces} analyzes noise amplification and order selection. Section~\ref{sec:experiments} validates empirically. Section~\ref{sec:conclusion} gives practical recommendations.


\section{Background and Related Work}
\label{sec:related}

Estimating log-determinants from trace powers connects classical moment theory, polynomial interpolation, and matrix computation.
We survey these connections before developing our framework in Section~\ref{sec:framework}.

\subsection{Problem Setup and Notation}
\label{sec:related-setup}

Let $A \in \R^{n \times n}$ be symmetric positive definite.
Its eigenvalues $\lambda_1 \leq \lambda_2 \leq \cdots \leq \lambda_n$ have condition number $\kappa = \lambda_n/\lambda_1$.

We have the first $m$ trace powers:
\[
p_k = \tr(A^k) = \sum_{i=1}^n \lambda_i^k, \qquad k = 1, 2, \ldots, m,
\]
and we want to estimate $\log\det(A) = \sum_{i=1}^n \log\lambda_i$.

Split the log-determinant into two parts.
Define the arithmetic and geometric means of the eigenvalues:
\begin{align*}
\AM &= p_1/n, \\
\GM &= (\det A)^{1/n}.
\end{align*}
Then
\begin{equation}
\label{eq:logdet-decomp}
\log\det(A) = n\log\AM + n\log(\GM/\AM).
\end{equation}
The first term follows directly from $p_1$.
The second term measures spectral deviation from uniformity.

\subsection{Elementary Symmetric Polynomials}
\label{sec:related-esp}

The GM/AM ratio in~\eqref{eq:logdet-decomp} involves a product of normalized eigenvalues. Elementary symmetric polynomials provide the classical machinery for relating power sums to such products.

Normalize the eigenvalues as $x_i = \lambda_i/\AM$ and define power sums $q_j = \sum_i x_i^j$.
Newton's identities relate these to elementary symmetric polynomials $e_k$.
For instance, $e_1 = n$ and $e_2 = (n^2 - q_2)/2$.
Continuing, $e_3 = (e_2 n - e_1 q_2 + q_3)/3$ and $e_4 = (e_3 n - e_2 q_2 + e_1 q_3 - q_4)/4$.
The normalized means $E_k = e_k/\binom{n}{k}$ appear in the bounds below.

\paragraph{Three-variable example.}
For positive numbers $a, b, c$, the elementary symmetric polynomials are $e_1 = a + b + c$, $e_2 = ab + bc + ca$, and $e_3 = abc$.
Newton's identities give $e_1 = q_1$, $e_2 = \tfrac{1}{2}(q_1^2 - q_2)$, $e_3 = \tfrac{1}{3}(e_2 q_1 - e_1 q_2 + q_3)$.
Thus $(q_1, q_2, q_3)$ determines $(e_1, e_2, e_3)$ and hence $abc$.

\paragraph{Large-$n$ problem.}
Computing $\det(A) = e_n$ via Newton's identities requires all $n$ power sums.
With only $m \ll n$ traces, exact computation is impossible.
What can we infer about $e_n$ (or $\log e_n$) from just $(q_1, \ldots, q_m)$?
Our exponent-space framework addresses this by estimating the limit of these identities as $n \to \infty$ without computing $e_n$ explicitly.

\subsection{Classical Moment Problems}
\label{sec:related-moments}

Our problem can be mapped to the classical moment problem by treating the eigenvalues as a discrete probability distribution. In this view, the trace powers $p_k = \sum \lambda_i^k$ correspond to the $k$-th moments $\E[X^k]$, and the log-determinant corresponds to $\E[\log X]$.

Classical results establish that a finite set of moments does not uniquely determine the underlying distribution~\citep{shohat1943,akhiezer1965}. Consequently, many spectra can share the first $m$ traces yet yield different log-determinants. This theoretical ambiguity underpins our impossibility result in Section~\ref{sec:impossibility}, which proves no estimator is uniform over all spectra, and motivates the moment-constrained bounds derived in Section~\ref{sec:bounds}.

\subsection{Classical Bounds on Geometric Means}
\label{sec:related-classical-bounds}

Given only finitely many moments, what can we guarantee about the geometric mean?
Three classical inequalities provide progressively tighter bounds.
\begin{itemize}
\item \textbf{Maclaurin's inequality.}
The AM--GM inequality generalizes to~\citep{hardy1952}:
\begin{equation}
\label{eq:maclaurin}
E_1 \geq E_2^{1/2} \geq E_3^{1/3} \geq \cdots \geq E_n^{1/n} = \GM.
\end{equation}
Thus $\GM \leq E_k^{1/k}$ using only $k$ power sums.

\item \textbf{Newton's inequalities.}
The sequence $(E_k)$ is log-concave~\citep{hardy1952,macdonald1995}:
\begin{equation}
\label{eq:newton-ineq}
E_k^2 \geq E_{k-1} E_{k+1}, \qquad k = 1, \ldots, n-1.
\end{equation}
Consequently, the slopes $s_k = \log E_k - \log E_{k-1}$ are nonincreasing.

\item \textbf{Tight mean-variance bound.}
Among spectra with fixed $(M_1, M_2)$, \citet{rodin2017} proved the geometric mean is maximized by a two-point distribution:
\begin{equation}
\label{eq:rodin}
\GM \leq (1 - d)^{(n-1)/n} (1 + (n-1)d)^{1/n}, \quad d = \sqrt{(M_2 - 1)/(n-1)}.
\end{equation}
This bound is sharp: equality holds for two-point spectra.

\end{itemize}

\subsection{Why Direct Moment Expansions Fail}
\label{sec:related-latane}

Rather than bounding the geometric mean, one might try to estimate it directly via Taylor expansion.
The volatility-drag formula expands $\log(\lambda)$ around $\AM$~\citep{latane1959,kelly1956}:
\begin{equation}
\label{eq:latane}
\E[\log X] \approx -\frac{\sigma^2}{2\AM^2} + \frac{\mu_3}{3\AM^3} - \frac{\mu_4}{4\AM^4} + \cdots
\end{equation}
with $\mu_k = \E[(\lambda - \AM)^k]$.
Our $k_{0:2}$ estimator (Section~\ref{sec:estimators}) recovers the leading term; higher orders give corrections.

For normalized eigenvalues with mean 1, this series converges only when $\lambda \in (0,2)$.
When $\kappa > 4$, some eigenvalues approach 0, and convergence becomes impractically slow.
Higher-moment extensions~\citep{jean1983} do not fix this.
Figure~\ref{fig:latane-comparison} shows the degradation for ill-conditioned spectra.


\subsection{Matrix-Vector Methods}
\label{sec:related-matvec}

A different computational paradigm avoids series expansions entirely by working with matrix-vector products.
Large-scale log-determinant estimation typically combines stochastic trace estimation with polynomial or rational approximations of $\log A$.
Hutchinson's estimator~\citep{hutchinson1990} approximates $\tr(B)$ by
\[
\tr(B) \approx \frac{1}{s} \sum_{i=1}^s z_i^\top B z_i
\]
with random vectors $z_i$ satisfying $\E[zz^\top] = I$.
Hutch++~\citep{meyer2021} improves sample complexity via low-rank deflation.
Since $\log\det(A) = \tr(\log A)$, the task reduces to estimating $\tr(f(A))$ with $f = \log$.

Two schemes approximate $\log A$:
\begin{itemize}
\item \textbf{Chebyshev methods}~\citep{han2015}: polynomial on $[\lambda_{\min}, \lambda_{\max}]$ needing $O(\sqrt{\kappa}\log(1/\epsilon))$ terms.
\item \textbf{Lanczos quadrature}~\citep{ubaru2017,saibaba2017}: stochastic Lanczos quadrature (SLQ) treats $\tr(f(A))$ as a Riemann-Stieltjes integral.
\end{itemize}

\paragraph{Data model.}
These methods require matrix-vector products $Av$ or shifted solves $(A - \sigma I)^{-1}v$.
We assume only traces $p_1, \ldots, p_m$ are available; no matrix-vector operations.

\begin{center}
\begin{tabular}{@{}lll@{}}
\toprule
& \textbf{Matrix-vector methods} & \textbf{Trace-power methods} \\
\midrule
Data & Products $Av$, solves & Traces $p_1, \ldots, p_m$ \\
Cost per query & $O(\text{nnz}(A))$ & $O(m)$ (precomputed traces) \\
Iterations & $O(\sqrt{\kappa})$ typical & $m \in \{4, \ldots, 8\}$ fixed \\
\bottomrule
\end{tabular}
\end{center}

\paragraph{Certification.}
If SLQ gives $\widehat{\log\det(A)}$ and trace powers are available, the interval $[L_k(r), U_k]$ from Section~\ref{sec:bounds} provides an independent certificate.
The extra cost is independent of $n$: closed-form bounds are $O(m)$, while optimization-based bounds solve a small ($k+1$-atom) problem.

\subsection{Information-Theoretic Motivation}
\label{sec:related-info}

Beyond matrix computation, log-determinants arise naturally in information theory and portfolio optimization.
The target $\E[\log \lambda] = \log(\GM)$ is the optimal growth rate for long-term wealth~\citep{kelly1956}.
\citet{breiman1961} proved that maximizing $\E[\log(\text{wealth})]$ asymptotically beats any other strategy.
In matrix terms, $\log\det(A) = n \cdot \E[\log \lambda]$ represents the total growth potential of eigenvalues as multiplicative factors.

\subsection{Open Questions}
\label{sec:related-gaps}

\begin{description}
\item[Trace-power regime (Sections~\ref{sec:framework}--\ref{sec:estimators}).]
Newton's identities need $n$ traces for exact $\det(A)$.
Can $p_1, \ldots, p_4$ give useful estimates when $m \ll n$?

\item[Failure detection (Section~\ref{sec:impossibility}).]
Latané's expansion converges slowly for $\kappa > 4$, but no test detects ill-conditioned spectra from observed traces.

\item[Fundamental limits (Section~\ref{sec:impossibility}).]
Moment theory proves non-uniqueness.
The achievable accuracy as a function of $m$ and spectral shape is unknown.

\item[Tighter bounds (Section~\ref{sec:bounds}).]
Maclaurin and Rodin bounds hold for any spectrum.
Can higher moments yield tighter guaranteed intervals, and how should point estimates combine with them?

\item[Noise sensitivity (Section~\ref{sec:noisy-traces}).]
Hutchinson's estimator introduces multiplicative noise in trace powers.
How does this noise propagate through the interpolation weights, and what order $m$ balances bias against variance?
\end{description}

\section{Cumulant Interpolation Framework}
\label{sec:framework}

Section~\ref{sec:related-latane} showed that $\lambda$-space expansions
diverge for $\kappa > 4$; adding more terms cannot rescue the
volatility-drag approach. The challenge is to use trace powers without falling into the same trap: trace powers $p_k = \tr(A^k) = \sum_i \lambda_i^k$ are dominated by large eigenvalues as $k$ grows, yet $\log\det A = \sum_i \log\lambda_i$ depends critically on the smallest eigenvalues. Any approach that treats traces at face value inherits this imbalance.

We work instead in exponent space: treat $t$ as a continuous exponent in
$M(t) = \E[X^t]$ and estimate $M'(0) = \E[\log X]$ by interpolating from
integer samples. Because $M(t)$ spans orders of magnitude, we transform to
$K(t) = \log M(t)$, which compresses the range and stabilizes interpolation.
This section develops the framework and formalizes the transform choice.
We first show that the choice of analytic transform is first-order equivalent: since all transforms share the same derivative at zero up to scaling, they all target the same log-sum quantity.

\begin{theorem}[Analytic transform reduction]
\label{thm:analytic-transform}
Assume $X>0$, $M(t)=\E[X^t]$ is analytic near $t=0$ (true for finite SPD spectra),
and $f$ is analytic on a neighborhood containing $\{M(0), M(1), \ldots, M(m)\}$
with $f'(1) \neq 0$. Define $G(t) = f(M(t))$. Then $G$ is analytic near $t=0$ and
\[
G'(0) = f'(1)\E[\log X].
\]
Consequently $\log\det A = n\bigl(\log\AM + G'(0)/f'(1)\bigr)$.
\end{theorem}

\begin{proof}
Since $M(0) = 1$ and $M'(0) = \E[\log X]$, the chain rule gives $G'(0) = f'(1)M'(0)$.
\end{proof}

All such transforms are equivalent to first order at $t=0$, so the practical choice
is about interpolation conditioning from integer samples. Any analytic $f$ yields
the same first-order target, but interpolation error and noise amplification depend
on the choice of $f$.

\paragraph{Why log?}
The Box--Cox family $f_\alpha(z) = (z^\alpha - 1)/\alpha$ interpolates between
identity ($\alpha = 1$) and log ($\alpha \to 0$). Appendix~\ref{app:transform-comparison}
and Table~\ref{tab:boxcox-comparison} show that while the optimal $\alpha$ varies with
spectrum, $\alpha = 0$ is consistently competitive, so we adopt it as our default.
The log transform is also the unique
variance-stabilizing choice under multiplicative noise (up to affine scaling) and
compresses the exponential growth of $M(t)$.

\subsection{Cumulant Generating Function}
\label{sec:cgf-def}

With the log transform chosen, define the cumulant generating function
\begin{equation}
K(t) = \log M(t).
\end{equation}
This compresses the exponential scale of $M(t)$ and yields additive trace ratios. For
computation we rescale eigenvalues by their arithmetic mean: $x_i = \lambda_i/\AM$,
and let $X$ be uniform on $\{x_1, \ldots, x_n\}$. Then
\begin{equation}
M(t) = \E[X^t] = \frac{1}{n} \sum_{i=1}^n x_i^t, \qquad M(1) = 1,
\end{equation}
so $K(0) = K(1) = 0$.
For integer $k \geq 1$, write $M_k := M(t)\big|_{t=k}$.
Differentiate: since $\frac{d}{dt}X^t\big|_{t=0} = \log X$,
\begin{equation}
\label{eq:Kprime-target}
K'(0) = \E[\log X] = \frac{1}{n}\sum_{i=1}^n \log x_i = \log\frac{\GM}{\AM},
\end{equation}
which is exactly the target. Thus
$\log\det A = n\bigl(\log\AM + K'(0)\bigr)$.
Traces give $K$ at integers $k \geq 2$:
\begin{equation}
\label{eq:K-from-traces}
M_k = \frac{n^{k-1} p_k}{p_1^k}, \qquad K(k) = \log\left(\frac{n^{k-1} p_k}{p_1^k}\right).
\end{equation}
We denote $\beta = M_2$, $\gamma = M_3$, $\delta = M_4$, so $K(2) = \log\beta$, etc.
Jensen's inequality gives $\beta \geq 1$ (equality only when all eigenvalues equal), so $K(2) \geq 0$. With $K(0) = K(1) = 0$ and $K'(0) < 0$, the curve rises from zero; interpolation must capture this shape.

\subsection{Interpolation Problem}
\label{sec:interp-problem}

From traces $p_1, \ldots, p_m$ we obtain $K(2), \ldots, K(m)$ plus anchors $K(0) = K(1) = 0$---a total of $m+1$ values at consecutive integers. Polynomial interpolation is natural: it is the unique lowest-degree fit, and differentiating polynomials is trivial. The task: fit degree-$m$ polynomial $P(t)$ through these points, read off $P'(0)$.

\subsection{Lagrange Form}
\label{sec:lagrange}

Lagrange's form expresses $P(t)$ as a weighted sum of data values---ideal for computing $P'(0)$:
\begin{equation}
P_m(t) = \sum_{j=0}^m K(j) L_j(t), \qquad L_j(t) = \prod_{k \neq j} \frac{t - k}{j - k}.
\end{equation}
Differentiate at zero: $P'_m(0) = \sum_{j=0}^m K(j) L'_j(0)$. Since $K(0) = K(1) = 0$, only nodes $j \geq 2$ contribute:
\begin{equation}
\label{eq:k0m-estimator}
\boxed{k_{0:m}: \quad \widehat{K'(0)} = \sum_{j=2}^{m} L'_j(0) \cdot K(j)}
\end{equation}
Section~\ref{sec:estimators} derives closed-form weights $L'_j(0)$; Appendix~\ref{app:interp-error} gives the standard interpolation error bound.

\paragraph{Summary.}
The CGF $K(t) = \log\E[X^t]$ converts trace data to cumulants, and $K'(0)$ equals the target $\log(\GM/\AM)$. Lagrange interpolation through $K(0), K(1), \ldots, K(m)$ yields the $k_{0:m}$ estimator: a linear combination of log-trace ratios with closed-form weights.

\section{The \texorpdfstring{$k_{0:m}$}{k0:m} Estimator Family}
\label{sec:estimators}

Section~\ref{sec:framework} shows we need $P'_m(0)$ to estimate the log-determinant. Computing these derivatives numerically or symbolically for each $m$ would be cumbersome. Fortunately, the weights admit a simple closed form.

\subsection{Closed-Form Weights}

\begin{theorem}[Lagrange derivative weights]
\label{thm:weights}
For nodes $\{0, 1, \ldots, m\}$:
\begin{equation}
\boxed{L'_j(0) = \frac{(-1)^{j-1}}{j}\binom{m}{j}, \quad j = 1, \ldots, m.}
\end{equation}
For $j = 0$: $L'_0(0) = -H_m$ where $H_m = \sum_{k=1}^{m} 1/k$ is the $m$-th harmonic number.
\end{theorem}

\begin{proof}
See Appendix~\ref{app:weights} for the full derivation. The case $j \geq 1$ uses
the quotient rule; the case $j = 0$ follows from the partition-of-unity property
$\sum_{j=0}^m L_j(t) = 1$ combined with a classical integral identity.
\end{proof}

\begin{remark}
The sign $(-1)^{j-1}$ alternates: $L'_2(0) < 0$, $L'_3(0) > 0$, $L'_4(0) < 0$, etc. This alternation enables error cancellation for smooth functions, the basis for convergence as $m$ increases.
\end{remark}

\subsection{Explicit Formulas}

Theorem~\ref{thm:weights} yields closed-form estimators. Define $\beta = M(2)$, $\gamma = M(3)$, $\delta = M(4)$ as in Section~\ref{sec:cgf-def}. The first three estimators admit equivalent representations:

\begin{center}
\small
\begin{tabular}{@{}llll@{}}
\toprule
& Cumulant form & Trace form & Moment form \\
\midrule
$k_{0:2}$ & $-\tfrac{1}{2}K(2)$ & $-\tfrac{1}{2}\log\tfrac{np_2}{p_1^2}$ & $\beta^{-1/2}$ \\[3pt]
$k_{0:3}$ & $-\tfrac{3}{2}K(2) + \tfrac{1}{3}K(3)$ & $-\tfrac{3}{2}\log\tfrac{np_2}{p_1^2} + \tfrac{1}{3}\log\tfrac{n^2p_3}{p_1^3}$ & $\beta^{-3/2}\gamma^{1/3}$ \\[3pt]
$k_{0:4}$ & $-3K(2) + \tfrac{4}{3}K(3) - \tfrac{1}{4}K(4)$ & $-3\log\tfrac{np_2}{p_1^2} + \tfrac{4}{3}\log\tfrac{n^2p_3}{p_1^3} - \tfrac{1}{4}\log\tfrac{n^3p_4}{p_1^4}$ & $\beta^{-3}\gamma^{4/3}\delta^{-1/4}$ \\
\bottomrule
\end{tabular}
\end{center}

\noindent The moment form exponentiates: $\widehat{\GM/\AM} = \beta^{w_2}\gamma^{w_3}\delta^{w_4}$ where $w_j = L'_j(0)$.

\paragraph{Worked example: $k_{0:4}$.} From Theorem~\ref{thm:weights}: $L'_2(0) = -\tfrac{1}{2}\binom{4}{2} = -3$, $L'_3(0) = +\tfrac{1}{3}\binom{4}{3} = +\frac{4}{3}$, $L'_4(0) = -\tfrac{1}{4}\binom{4}{4} = -\frac{1}{4}$. Thus:
\begin{equation}
\label{eq:k04}
\boxed{k_{0:4}: \quad \widehat{K'(0)} = -3 K(2) + \frac{4}{3} K(3) - \frac{1}{4} K(4) = \log\bigl(\beta^{-3}\gamma^{4/3}\delta^{-1/4}\bigr)}
\end{equation}
The simplest estimator $k_{0:2}\colon \widehat{K'(0)} = -\tfrac{1}{2}K(2) = -\tfrac{1}{2}\log\beta$ is the exact form of Latan\'e's volatility-drag approximation. Appendix~\ref{app:weights-table} lists weights for $m = 2, \ldots, 8$; Section~\ref{sec:experiments} shows when higher orders improve upon $k_{0:4}$.

\subsection{Weight Growth and Stability}

\begin{proposition}[Weight magnitude growth]
For fixed $j \geq 2$, as $m \to \infty$:
\[
|L'_j(0)| = \frac{1}{j}\binom{m}{j} = O(m^j/j!).
\]
For the highest term ($j = m$): $|L'_m(0)| = 1/m$.
\end{proposition}

Weights grow as $m^j$ for fixed $j$, but alternating signs cancel errors for smooth functions. For non-smooth $K(t)$, cancellation fails and errors grow. This instability is a manifestation of the Runge phenomenon, where high-degree polynomial interpolation oscillates wildly at interval edges.

For noisy trace inputs (Section~\ref{sec:noisy-traces}), these weight magnitudes directly determine noise amplification. The $\ell_2$ norm $\lVert w \rVert_2 = \sqrt{\sum w_k^2}$ grows as $2^m/m^{5/4}$, creating an exponential penalty for high interpolation orders when traces have even modest relative noise ($\eta \gtrsim 1\%$).

\noindent The end-to-end estimator is summarized in Algorithm~\ref{alg:k0m}. The only
matrix-dependent work is computing traces; all remaining steps scale with $m$.

\begin{algorithm}[t]
\caption{$k_{0:m}$ Log-Determinant Estimator}
\label{alg:k0m}
\begin{algorithmic}[1]
\Require SPD matrix $A \in \R^{n \times n}$ (implicit), order $m$
\Ensure Estimate $\widehat{\log\det A}$
\State Compute traces $p_1 = \tr(A), \ldots, p_m = \tr(A^m)$
\State $\mu \gets p_1 / n$ \Comment{Arithmetic mean}
\For{$k = 2, \ldots, m$}
    \State $K(k) \gets \log(n^{k-1} p_k / p_1^k)$
\EndFor
\State Compute weights $L'_j(0) = (-1)^{j-1}\binom{m}{j}/j$ for $j = 2, \ldots, m$
\State $\widehat{K'(0)} \gets \sum_{j=2}^{m} L'_j(0) \cdot K(j)$
\State \Return $n \cdot (\log\mu + \widehat{K'(0)})$
\end{algorithmic}
\end{algorithm}

Algorithm~\ref{alg:k0m} applies to any order $m$, and the only matrix-dependent
step is computing traces; all other operations are $O(m)$. When is the estimate
exact? The next section identifies the spectral conditions under which $k_{0:m}$
recovers $K'(0)$ without error.

\section{When Is the \texorpdfstring{$k_{0:m}$}{k0:m} Estimator Exact?}
\label{sec:exactness}

With the $k_{0:m}$ weights in hand, we now ask: when is the estimator exact? The answer identifies the best-case spectra and reveals the boundary beyond which all trace-based methods must approximate.

\subsection{Polynomial Exactness}

\begin{proposition}[Polynomial exactness]
\label{prop:poly-exact}
If $K(t)$ is a polynomial of degree at most $m$, then $k_{0:m}$ is exact: $\widehat{K'(0)} = K'(0)$.
\end{proposition}

\begin{proof}
Lagrange interpolation on $m+1$ nodes reproduces polynomials of degree $\le m$ exactly. Differentiate at $t=0$ to recover $K'(0)$.
\end{proof}

Which eigenvalue distributions make $K(t)$ polynomial? Only one non-trivial family: lognormal.

\begin{theorem}[Lognormal exactness]
\label{thm:lognormal}
Let $\log X \sim \mathcal{N}(\mu, \sigma^2)$. Then $K(t) = \mu t + \sigma^2 t^2/2$ is quadratic, and $k_{0:m}$ is exact for all $m \ge 2$.
\end{theorem}

\begin{proof}
For $Y = \log X \sim \mathcal{N}(\mu, \sigma^2)$, the moment-generating function is $\E[e^{tY}] = e^{\mu t + \sigma^2 t^2/2}$. Since $X = e^Y$, we have $\E[X^t] = \E[e^{tY}]$, so $K(t) = \mu t + \sigma^2 t^2/2$. A quadratic is reproduced exactly by any interpolant of degree $\ge 2$.
\end{proof}

\begin{remark}[Uniqueness]
Why is lognormal special? If $X$ is lognormal, then $Y = \log X$ is Gaussian and its cumulant generating function is $K_Y(t) = \mu t + \sigma^2 t^2/2$. Gaussian distributions are uniquely characterized by having all cumulants of order $\ge 3$ vanish, so the CGF is exactly quadratic. This quadratic form is inherited by $K(t) = \log \E[X^t] = \log \E[e^{tY}]$. Among non-degenerate distributions, only lognormal yields quadratic $K(t)$.
\end{remark}

\paragraph{Closed form for lognormal spectra.}
If $K(t)$ is quadratic, skip $k_{0:m}$. From $K(t) = \mu t + \sigma^2 t^2/2$, we get $K'(0) = \mu$. Normalization $\E[X] = 1$ gives $K(1) = 0$, so $\mu = -\sigma^2/2$ and $K(2) = \sigma^2$. Hence $\GM/\AM = e^{-\sigma^2/2} = 1/\sqrt{e^{K(2)}}$, yielding:
\begin{equation}
\label{eq:lognormal-closed}
\widehat{\det(A)^{1/n}} = \frac{\tr(A)^2}{n\sqrt{n \cdot \tr(A^2)}}.
\end{equation}
This needs only two traces, half the data of $k_{0:4}$.

\begin{remark}[Connection to Latan\'e]
\label{rem:latane}
Equation~\eqref{eq:lognormal-closed} is the exact version of Latan\'e's volatility-drag approximation~\citep{latane1959}. Latan\'e used Taylor expansion; ours is exact for any lognormal spectrum.
\end{remark}

\paragraph{Finite-$n$ plug-in behavior.}
Theorem~\ref{thm:lognormal} describes the population. Model eigenvalues as i.i.d.\ lognormal draws. The empirical traces converge to population moments as $n \to \infty$; continuity of the estimator (Lemma~\ref{lem:plugin-continuity}) then implies convergence to the population geometric mean.

\begin{lemma}[Continuity of plug-in estimator]
\label{lem:plugin-continuity}
Fix $m \ge 2$ and $n$. On $\{(p_1, \ldots, p_m) : p_k > 0\}$, the map
\[
(p_1, \ldots, p_m) \mapsto n^{-H_m}\,p_1^{\,m}\prod_{k=2}^m p_k^{\,w_k},
\qquad w_k = \frac{(-1)^{k-1}}{k}\binom{m}{k}, \quad H_m = \sum_{j=1}^m \frac{1}{j},
\]
is continuous. Hence $p_k/n \to \mu_k$ implies convergence of the trace-based estimator.
\end{lemma}

\begin{proof}
Each power $p_k \mapsto p_k^{\,w_k}$ is continuous on $(0,\infty)$; multiplying by the constant $n^{-H_m}$ and taking finite products preserves continuity.
\end{proof}

\paragraph{Empirical validation.}
We drew 1000 lognormal spectra ($\sigma = 0.5$, $\kappa \approx 12$) for each $n$ and ran $k_{0:4}$. Bias stayed near zero; RMSE decayed as $O(n^{-0.5})$, the Monte Carlo rate. Figure~\ref{fig:scaling} confirms both exactness and consistency.

\begin{widefig}[t]
\centering
\includegraphics[width=\textwidth]{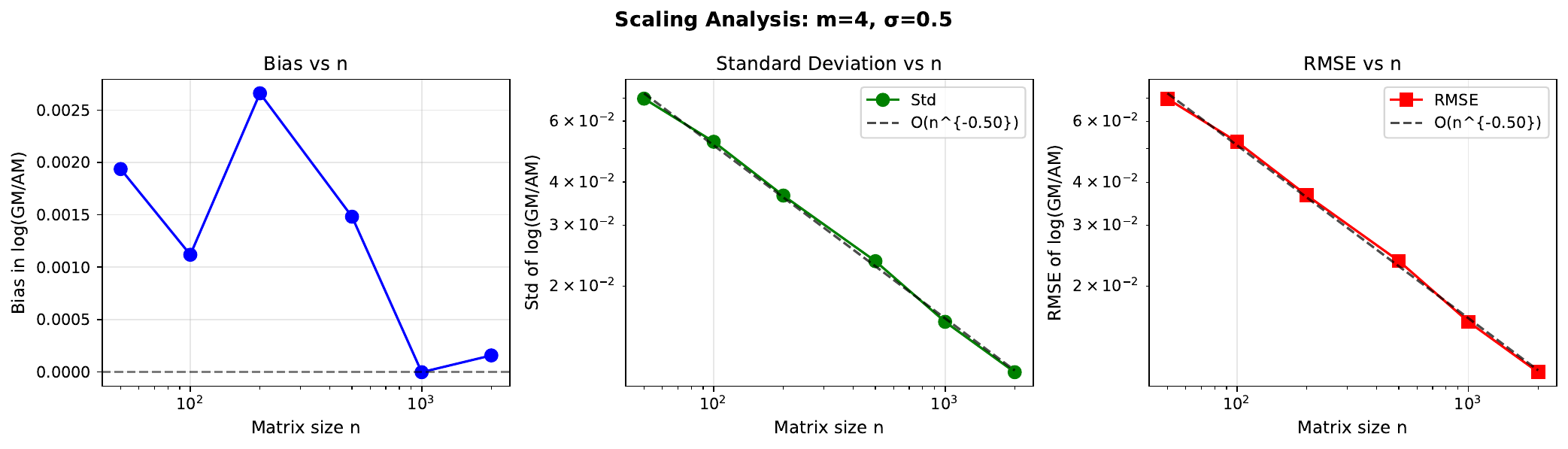}
\caption{Scaling with $n$ for lognormal eigenvalues ($\sigma = 0.5$). Left: bias $\approx 0$. Middle/Right: SD and RMSE decay as $O(n^{-0.5})$.}
\label{fig:scaling}
\end{widefig}

\begin{remark}[Practical guidance]
Lognormal marks the exact/approximate boundary. If log-eigenvalues look roughly Gaussian, formula~\eqref{eq:lognormal-closed} gives a quick diagnostic. Outside this model, $k_{0:m}$ is approximate; the exactness analysis shows why.
\end{remark}

Lognormal spectra thus mark the exactness boundary. Beyond this case, $k_{0:m}$ approximates. But how badly can it approximate? The next section shows that errors can be arbitrarily severe: no trace-based method using finitely many moments can succeed uniformly over all spectra.

\section{Why No Estimator Can Succeed Universally}
\label{sec:impossibility}

Section~\ref{sec:exactness} identified when $k_{0:m}$ is exact---a best-case scenario.
But most spectra are not lognormal. How badly can the estimator fail? If errors were
always modest, practitioners could accept $k_{0:m}$ as ``good enough.'' If errors can
be arbitrarily large, we need safeguards. We now prove that no finite-moment estimator
can be uniformly accurate: the information gap is a fundamental barrier.

\subsection{Setup}

Let $A \in \R^{n \times n}$ be SPD with eigenvalues $\lambda_1, \ldots, \lambda_n$ and mean $\AM = \frac{1}{n}\sum_i \lambda_i$. Define $x_i = \lambda_i/\AM$ and draw $X$ uniformly from $\{x_i\}$. Then $M(t) = \E[X^t] = \frac{1}{n}\sum_i x_i^t$ and $K(t) = \log M(t)$ (the cumulant generating function from Section~\ref{sec:cgf-def}). The target is $K'(0) = \E[\log X] = \log(\GM/\AM)$, giving $\log\det A = n(\log\AM + K'(0))$.
We observe $K(t)$ only at $m$ positive nodes $T = \{t_1, \ldots, t_m\} \subset (0, \infty)$.

\subsection{Two Obstructions}

Positive moments shrink small eigenvalues: $x^t \to 0$ as $x \to 0$ for $t > 0$. So $M(t)$ can stay finite while the smallest eigenvalue vanishes. Yet $K'(0) = \E[\log X]$ diverges like $\log(\min_i x_i)$. This mismatch yields two distinct barriers:

\begin{itemize}
\item \textbf{Obstruction 1: Non-identifiability.} Different spectra can share identical moments yet differ in $K'(0)$. Theorem~\ref{thm:impossibility}(a) constructs explicit examples; this is an instance of the classical moment problem.
\item \textbf{Obstruction 2: Saturation.} Even continuous estimators must fail: moments can converge to finite limits while $K'(0) \to -\infty$. Theorem~\ref{thm:impossibility}(b) proves this for any continuous estimator.
\end{itemize}

Can any finite set of moments determine $K'(0)$ universally? The answer is no:

\begin{theorem}[Impossibility]
\label{thm:impossibility}
Fix any finite $T = \{t_1, \ldots, t_m\} \subset (0, \infty)$.

\paragraph{(a) Non-identifiability.}
There exist spectra $\mu, \nu$ with $K_\mu(t_j) = K_\nu(t_j)$ for all $j$ but $K'_\mu(0) \neq K'_\nu(0)$. Hence no function of $(K(t_1), \ldots, K(t_m))$ can always return $K'(0)$. (This is an instance of the classical moment problem; see~\citet{shohat1943} and~\citet{akhiezer1965}.)

\paragraph{(b) Saturation.}
For any continuous estimator $\widehat{K'_T} = f(K(t_1), \ldots, K(t_m))$, there exists a sequence $A_\kappa$ with
\[
\lim_{\kappa \to \infty} K_{A_\kappa}(t_j) \text{ finite for each } j, \quad \text{but} \quad \lim_{\kappa \to \infty} K'_{A_\kappa}(0) = -\infty.
\]
So $\widehat{K'_T}$ stays bounded while the truth diverges.
\end{theorem}

\begin{proof}
\textbf{(a)} Pick $m+3$ distinct points $x_0, \ldots, x_{m+2} > 0$. The constraint matrix
\[
V = \begin{bmatrix}
1 & \cdots & 1 \\
x_0 & \cdots & x_{m+2} \\
x_0^{t_1} & \cdots & x_{m+2}^{t_1} \\
\vdots & & \vdots \\
x_0^{t_m} & \cdots & x_{m+2}^{t_m}
\end{bmatrix}
\]
has a nonzero null vector $v$. For small $\varepsilon > 0$, set $w^{(\pm)} = w^{(0)} \pm \varepsilon v$ where $w^{(0)} > 0$ satisfies the mean-one constraint. The measures $\mu = \sum_i w_i^{(+)} \delta_{x_i}$ and $\nu = \sum_i w_i^{(-)} \delta_{x_i}$ share all moments in $T$, yet $K'_\mu(0) - K'_\nu(0) = 2\varepsilon \sum_i v_i \log x_i \neq 0$ generically.

\textbf{(b)} Take $A_\kappa = \diag(1, \kappa)$. Then
\[
M_\kappa(t) = \tfrac{1}{2}\bigl[(2/(1+\kappa))^t + (2\kappa/(1+\kappa))^t\bigr] \to 2^{t-1},
\]
so $K_\kappa(t) \to (t-1)\log 2$. But $K'_\kappa(0) = \log\frac{2\sqrt{\kappa}}{1+\kappa} \to -\infty$. A continuous $f$ keeps $\widehat{K'_T}$ finite, so the error grows without bound.
\end{proof}

\begin{example}[Nodes $\{2, 4\}$]
\label{ex:nonident24}
With support $\{0.1, 0.5, 1, 2, 10\}$, two weight vectors give identical $\E[X] = 1$, $\E[X^2]$, $\E[X^4]$, but $\E[\log X]$ differs by 5\%.
\end{example}

\subsection{Taylor Radius}

The saturation obstruction can be sharpened by locating singularities of $K(t)$.
Two-point spectra provide the sharpest bounds on this radius, making them the critical worst-case examples.

\begin{proposition}[Taylor radius for two-point spectra]
\label{prop:taylorradius}
Let $X$ take values $x_1, x_2$ with weights $(1-p), p$ and ratio $\kappa = x_2/x_1 > 1$. The Taylor series of $K(t)$ about $t = 0$ has radius
\[
\boxed{R = \frac{1}{\log\kappa}\sqrt{\log^2\bigl(\tfrac{1-p}{p}\bigr) + \pi^2}.}
\]
For equal weights ($p = 1/2$): $R = \pi/\log\kappa$.
\end{proposition}

\begin{proof}
Singularities of $K(t) = \log M(t)$ occur at zeros of $M(t) = (1-p)x_1^t + px_2^t$. Setting $(1-p) + p\kappa^t = 0$ gives $t_k = [\log((1-p)/p) + i(2k+1)\pi]/\log\kappa$. The radius equals the distance to the nearest zero.
\end{proof}

Geometrically, $R$ marks the convergence boundary: polynomial interpolation at nodes $t > R$ extrapolates beyond the analytic domain, causing divergence.

\begin{corollary}
For equal-weight two-point spectra with $\kappa > e^{\pi/2} \approx 4.81$, we have $R < 2$, so the Taylor series diverges at $t = 2$.
\end{corollary}

\begin{figure}[t]
\centering
\includegraphics[width=0.85\columnwidth]{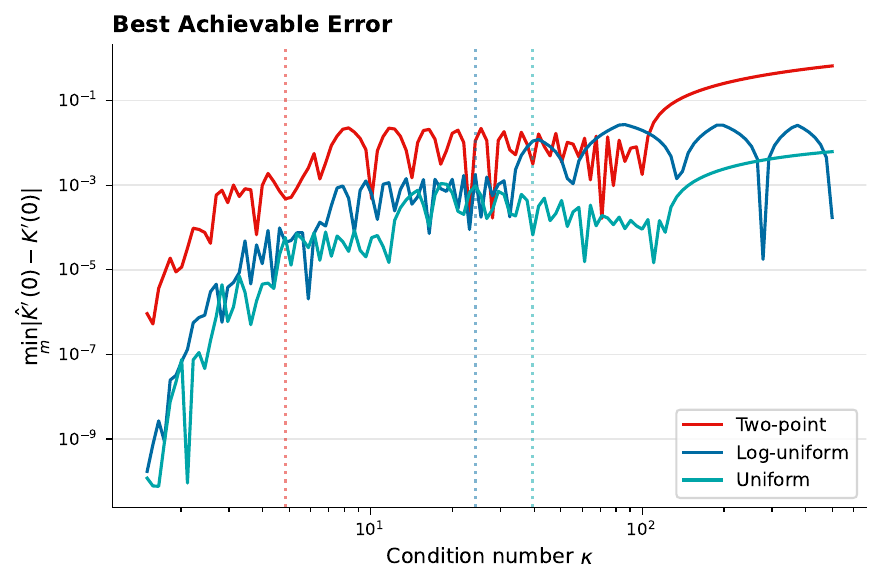}
\caption{Best achievable error in log-geometric-mean estimation.
For each condition number $\kappa$, we compute the $k_{0:m}$ estimator
for $m = 2, \ldots, 20$ and report the minimum absolute error
$|\hat{K}'(0) - K'(0)|$ over all $m$.
Two-point spectra (red) have the smallest Taylor radius
$R = \pi/\log\kappa$ and exhibit the highest error floor;
log-uniform spectra (blue) have $R = 2\pi/\log\kappa$;
uniform spectra (cyan) have the largest $R$ and lowest floor.
Vertical dotted lines mark $\kappa$ thresholds where $R$ crosses integer values.
Even with oracle choice of $m$, spectrum type determines fundamental accuracy limits.}
\label{fig:taylor-radius-best}
\end{figure}

\subsection{Practical Lesson}

No estimator works for every spectrum. Ask instead: which spectra are safe? \emph{Smooth} spectra (geometric, uniform) have large Taylor radius and permit accurate interpolation. \emph{Outlier-dominated} spectra (two-point, bimodal) have small Taylor radius and cause failure (Figure~\ref{fig:taylor-radius-best}). These impossibility results are not curiosities---they imply that any point estimator needs backup. Section~\ref{sec:bounds} provides such certificates: bounds that remain valid even when point estimates fail, with interval width signaling reliability.


\section{Guaranteed Bounds}
\label{sec:bounds}

The impossibility results of Section~\ref{sec:impossibility} are sobering: no
point estimator can uniformly succeed across all spectra. But impossibility
results describe worst cases; they do not preclude useful guarantees. This raises
a practical question: \emph{given that point estimates may fail, what can we still
certify?}

The answer is \emph{bounds}. Unlike point estimates, which may be arbitrarily wrong
for adversarial spectra, bounds provide deterministic guarantees: the true geometric
mean lies within $[L_k(r), U_k]$ for \emph{any} spectrum consistent with observed
moments.

Classical bounds exist---Maclaurin's inequality, Newton's inequalities, and Rodin's two-point bound (Section~\ref{sec:related-classical-bounds})---but they either require structure we lack or leave room for improvement.
Our contribution is twofold: (i) \emph{new bounds} that exploit log-concavity (last-slope) and moment constraints (feasible-set optimization), and (ii) a \emph{measure-theoretic framework} revealing that optimal bounds are achieved by $(k+1)$-atom distributions.
This second insight is the key: it reduces bound computation from an $n$-dimensional search over eigenvalues to a fixed $(2k+2)$-dimensional optimization---certification cost is independent of matrix size.

Narrow intervals pin down the answer; wide intervals diagnose unreliability.
The decision rule is simple: clip the estimate to $[L, U]$ if outside, and always
report both the estimate and interval (Section~\ref{sec:bounds-diagnostic}).

\subsection{Last-Slope Bound}
\label{sec:bounds-lastslope}

Maclaurin's inequality~\eqref{eq:maclaurin} gives $\GM \leq E_m^{1/m}$.
We improve this by exploiting log-concavity of the $E_j$ sequence,
guaranteed by Newton's inequalities~\eqref{eq:newton-ineq}, which imply $E_{j-1} E_{j+1} \le E_j^2$.

Define $g(j) = \log E_j$ and slopes $s_j = g(j) - g(j-1)$.
Log-concavity implies $s_j$ is nonincreasing in $j$.

\begin{theorem}[Last-slope bound]
\label{thm:last-slope}
For $m < n$:
\begin{equation}
\boxed{U_{\mathrm{LS},m} = \left[ E_m \left( \frac{E_m}{E_{m-1}} \right)^{n-m} \right]^{1/n}}
\end{equation}
\end{theorem}

\begin{proof}
Since $s_j \leq s_m$ for $j \geq m$:
$g(n) = g(m) + \sum_{j=m+1}^{n} s_j \leq g(m) + (n-m) s_m$.
Exponentiating: $E_n \leq E_m (E_m/E_{m-1})^{n-m}$.
Taking the $n$-th root and using $\GM(x) = E_n^{1/n}$ yields the bound.
\end{proof}

\begin{remark}[Dominance over Maclaurin]
The last-slope bound $U_{\mathrm{LS},m}$ is strictly tighter than Maclaurin's $E_m^{1/m}$ whenever
$E_m/E_{m-1} < E_m^{1/m}$, which holds for most spectra. The improvement comes from
extrapolating the log-concave decay rather than truncating at $E_m$.
However, the optimization-based bounds $U_k$ (Section~\ref{sec:bounds-compact}) are typically tighter still; last-slope's value is computational simplicity (closed-form, no optimization).
\end{remark}

\paragraph{Combined closed-form bound.}
Taking the minimum of all available bounds (Maclaurin~\eqref{eq:maclaurin},
Rodin~\eqref{eq:rodin}, and last-slope):
\begin{equation}
\label{eq:combined-bound}
\boxed{U_{\mathrm{closed}} = \min\left\{ U_{\mathrm{Rodin}}, \; E_4^{1/4}, \; U_{\mathrm{LS},4} \right\}}
\end{equation}

\subsection{Feasible Sets and Bound Definitions}
\label{sec:bounds-defs}

We work with normalized eigenvalues $x_i = \lambda_i/\AM$ satisfying $M_1 = 1$.
Bounds on $\GM(x) = (\prod_i x_i)^{1/n}$ yield bounds on $\det(A)^{1/n} = \AM \cdot \GM(x)$.
The power sums $q_j = \sum_i x_i^j = n \cdot M_j$ determine the elementary symmetric
polynomials $e_k$ via Newton's identities (Section~\ref{sec:related}).

To formalize these bounds computationally for general $k$, we define the sets of spectra consistent with observed moments.

We define both the exact finite-$n$ feasible set $\mathcal{F}_{n,k}$ and its continuous relaxation $\mathcal{F}_k$.

\begin{definition}[Feasible sets]
\label{def:feasible-sets}
For normalized moments $M_1 = 1, M_2, \ldots, M_k$, the \emph{finite-$n$ feasible set} is:
\[
\mathcal{F}_{n,k} = \left\{ x \in (0,\infty)^n : \frac{1}{n}\sum_{t=1}^n x_t^j = M_j, \; j=1,\ldots,k \right\}
\]
The \emph{measure relaxation} is the set of probability measures on $(0,\infty)$ matching these moments:
\[
\mathcal{F}_k = \left\{ \mu : \int x^j \, d\mu = M_j, \; j=1,\ldots,k \right\}
\]
For lower bounds with a minimum eigenvalue constraint $r > 0$, we define the restricted set $\mathcal{F}_k(r) = \{\mu \in \mathcal{F}_k : \mathrm{supp}(\mu) \subseteq [r,\infty)\}$.
\end{definition}

\begin{definition}[$k$-Trace bounds]
\label{def:ktrace-bounds}
\begin{align}
U_{n,k} &= \max_{x \in \mathcal{F}_{n,k}} \left(\prod_{t=1}^n x_t\right)^{1/n}
    & \text{(finite-}n\text{ upper bound)} \\
L_{n,k}(r) &= \min_{x \in \mathcal{F}_{n,k} \cap [r,\infty)^n} \left(\prod_{t=1}^n x_t\right)^{1/n}
    & \text{(finite-}n\text{ lower bound)} \\
U_k &= \sup_{\mu \in \mathcal{F}_k} \exp\left(\int \log x \, d\mu\right)
    & \text{(relaxed upper bound)} \\
L_k(r) &= \inf_{\mu \in \mathcal{F}_k(r)} \exp\left(\int \log x \, d\mu\right)
    & \text{(relaxed lower bound)}
\end{align}
\end{definition}

\subsection{Compactness and Finite Support}
\label{sec:bounds-compact}

One might attempt to find extremal distributions by optimizing over all $n$ eigenvalues---a problem with $O(n)$ variables.
However, since we only constrain $k$ moments, the problem allows for a massive complexity reduction.
We show that optimal measures have at most $k+1$ atoms, collapsing the search space from $\R^n$ to $\R^{2(k+1)}$ (atom locations and weights).
For example, if $n = 10^6$ and $k = 4$, the optimization involves just 10 variables rather than $10^6$.

\begin{theorem}[Compactness and finite support]
\label{thm:compact-support}
Assume $M_k < \infty$ for some $k \geq 1$.
\begin{enumerate}
\item The set $\mathcal{F}_k$ is convex and compact in the weak topology.
\item Every extreme point of $\mathcal{F}_k$ is a discrete measure with
      \textbf{at most $k+1$ atoms}.
\item Consequently, $U_k$ and $L_k(r)$ admit optimizers supported on at most
      $k+1$ points.
\end{enumerate}
\end{theorem}

\begin{proof}
\emph{(i) Tightness:} For any $T > 0$, Markov's inequality gives
$\mu([T,\infty)) \leq M_k/T^k$
uniformly over $\mu \in \mathcal{F}_k$. Hence $\mathcal{F}_k$ is tight.

\emph{Closedness:} If $\mu_n \Rightarrow \mu$ weakly and $\sup_n \int x^k \, d\mu_n < \infty$,
then by uniform integrability, $\int x^j \, d\mu_n \to \int x^j \, d\mu$ for $j \leq k$.
By Prokhorov's theorem, tight $+$ closed implies compact.

\emph{(ii) Extreme points:} We have $k+1$ linear constraints: $\int 1 \, d\mu = 1$
(normalization) and $\int x^j \, d\mu = M_j$ for $j = 1, \ldots, k$. By Richter's
theorem on moment spaces~\citep{richter1957}, extreme points have at most $k+1$ atoms.

\emph{(iii) Existence:} The objective $\mu \mapsto \int \log x \, d\mu$ is linear.
Existence follows from compactness; for $L_k(r)$ with $r > 0$, $\log x$ is bounded
below on $[r,\infty)$.
\end{proof}

\begin{corollary}[Monotonicity]
\label{cor:monotonicity}
$U_{k+1} \leq U_k$ and $L_{k+1}(r) \geq L_k(r)$.
Additional constraints shrink feasible sets.
\end{corollary}

\begin{remark}[Computational complexity]
\label{rem:complexity}
The $(k+1)$-atom structure implies $O(k^3)$ complexity for solving the bounds optimization (Definition~\ref{def:ktrace-bounds}), independent of matrix dimension $n$.
This arises from measure-theoretic compactness: the space of distributions matching $k$ moments is weak-* compact, and extreme points are finitely supported.
Thus certification cost scales with the number of available traces, not the matrix size.
\end{remark}

\subsection{Relaxation Validity}
\label{sec:bounds-validity}

\begin{theorem}[Relaxation provides valid bounds]
\label{thm:relaxation-valid}
For any $n$ and any feasible spectrum $x \in \mathcal{F}_{n,k}$:
\begin{enumerate}
\item $\GM(x) \leq U_{n,k} \leq U_k$ \quad (upper bound valid)
\item If $x \in [r,\infty)^n$: $L_k(r) \leq L_{n,k}(r) \leq \GM(x)$ \quad (lower bound valid)
\end{enumerate}
\end{theorem}

\begin{proof}
The empirical measure $\mu_x = \frac{1}{n}\sum_{t=1}^n \delta_{x_t}$ lies in
$\mathcal{F}_k$ and satisfies
$\exp(\int \log x \, d\mu_x) = (\prod_{t=1}^n x_t)^{1/n} = \GM(x)$.
Taking supremum over $\mathcal{F}_k$ gives $\GM(x) \leq U_k$. The inequalities
$U_{n,k} \leq U_k$ and $L_k(r) \leq L_{n,k}(r)$ follow since finite-$n$ empirical
measures are feasible for the relaxation.
\end{proof}

\begin{corollary}[Quantization gap under bounded support]
\label{cor:quantization}
If the optimizer of $U_k$ has support in $[r, R]$ for some $0 < r \leq R < \infty$, then
$0 \leq U_k - U_{n,k} \leq C(r,R,k)/n$
for a constant depending on $\max\{|\log r|, |\log R|\}$.
\end{corollary}

\subsection{Lower Bounds with Eigenvalue Constraints}
\label{sec:bounds-lower}

The upper bound $U_k$ uses only moment information. A lower bound requires
additional structure: a floor $r \leq \lambda_{\min}/\AM$ on the smallest
normalized eigenvalue.

\paragraph{When is $r$ available?}
\begin{itemize}
\item \emph{Regularized problems:} $A + \sigma I$ has $\lambda_{\min} \geq \sigma$,
      so $r = \sigma/\AM$.
\item \emph{Gershgorin bounds:} $r = \min_i(A_{ii} - \sum_{j \neq i}|A_{ij}|)/\AM$.
\item \emph{Physical constraints:} Eigenvalues represent positive quantities with
      known lower limits.
\end{itemize}

\begin{theorem}[Lower bound structure]
\label{thm:lower-structure}
For $r > 0$, there exists an optimizer $\mu^*$ of $L_k(r)$ that is discrete with
$\leq k+1$ atoms and satisfies $r \in \mathrm{supp}(\mu^*)$.
\end{theorem}

\begin{proof}[Proof sketch]
Minimizing $\int \log x \, d\mu$ over $\mathcal{F}_k(r)$ pushes mass toward small $x$.
The boundary $x = r$ is always active at the optimum. By Theorem~\ref{thm:compact-support},
the optimizer has $\leq k+1$ atoms.
\end{proof}

\paragraph{Closed form for $k = 2$.}
With two moment constraints, the optimizer has exactly two atoms at $x_1 = r$ and $x_2 > r$:
\begin{equation}
w_1^* = \frac{M_2 - 1}{(r - 1)^2 + (M_2 - 1)}, \quad
x_2^* = \frac{1 - w_1^* r}{1 - w_1^*}, \quad
L_2(r) = r^{w_1^*} \cdot (x_2^*)^{1 - w_1^*}
\end{equation}

\subsection{Computing $k$-Trace Bounds}
\label{sec:bounds-algorithms}

By Theorem~\ref{thm:compact-support}, the optimizer has at most $k+1$ support points.
By Corollary~\ref{cor:monotonicity}, $U_k$ is the tightest bound using $k$ moments,
so we solve directly for $k+1$ atoms without searching over smaller support sizes.

\noindent The bound computations reduce to small moment-matching problems. We summarize
the upper and lower procedures in Algorithms~\ref{alg:nlp-upper}--\ref{alg:nlp-lower}.

\begin{algorithm}[t]
\caption{$k$-Trace Upper Bound}
\label{alg:nlp-upper}
\begin{algorithmic}[1]
\Require Normalized moments $M_1 = 1, M_2, \ldots, M_k$; dimension $n$
\Ensure Upper bound $U_k \geq \GM(x)$
\If{$k = 2$} \Comment{Rodin's closed form~\eqref{eq:rodin}}
    \State $d \gets \sqrt{(M_2 - 1)/(n-1)}$
    \State \Return $(1-d)^{(n-1)/n}(1+(n-1)d)^{1/n}$
\EndIf
\State Solve with $k+1$ support points:
\begin{align*}
\max_{w, x} &\sum_{i=1}^{k+1} w_i \log x_i \\
\text{s.t.} \quad &\sum_i w_i = 1, \; \sum_i w_i x_i^j = M_j \; (j=1,\ldots,k),\\
&w_i \geq 0, \; x_i > 0
\end{align*}
\State \Return $\exp(\text{optimal objective})$
\end{algorithmic}
\end{algorithm}

\begin{algorithm}[t]
\caption{$k$-Trace Lower Bound}
\label{alg:nlp-lower}
\begin{algorithmic}[1]
\Require Normalized moments $M_1 = 1, M_2, \ldots, M_k$; floor $r > 0$
\Ensure Lower bound $L_k(r) \leq \GM(x)$
\If{$k = 2$} \Comment{Closed form}
    \State $w_1 \gets (M_2 - 1)/((r-1)^2 + (M_2 - 1))$
    \State $x_2 \gets (1 - w_1 r)/(1 - w_1)$
    \State \Return $r^{w_1} \cdot x_2^{1-w_1}$
\EndIf
\State Solve with $k+1$ support points (one fixed at $r$):
\begin{align*}
\min_{w, x} &\sum_{i=1}^{k+1} w_i \log x_i \\
\text{s.t.} \quad &\sum_i w_i = 1, \; \sum_i w_i x_i^j = M_j \; (j=1,\ldots,k),\\
&w_i \geq 0, \; x_i \geq r
\end{align*}
\State \Return $\exp(\text{optimal objective})$
\end{algorithmic}
\end{algorithm}

\begin{remark}[Implementation]
For $k \leq 4$, the closed-form bounds often suffice and require no solver.
Algorithms~\ref{alg:nlp-upper}--\ref{alg:nlp-lower} are needed for $k > 4$ or when
tighter NLP-based bounds are desired. Multistart local optimization typically
finds the global optimum; for certified guarantees, use global solvers.
\end{remark}

These algorithms make the structure explicit: once the moments are known, the
remaining optimization is low-dimensional (at most $k+1$ support points) and
independent of the matrix size $n$.

\subsection{Certified Log-Determinant Estimation}
\label{sec:bounds-certified}

\begin{theorem}[Certified interval]
\label{thm:certified-interval}
Let $A \succ 0$ with $\AM = p_1/n$ and normalized moments $M_j$. If $r \leq \lambda_{\min}/\AM$:
\[
n \log(\AM) + n \log L_k(r) \leq \log\det(A) \leq n \log(\AM) + n \log U_k
\]
\end{theorem}

In practice, combine the point estimate from Algorithm~\ref{alg:k0m} with the bounds
from Algorithms~\ref{alg:nlp-upper}--\ref{alg:nlp-lower}, taking the tightest upper
bound among available closed-form options.

\subsection{Gap Diagnostic}
\label{sec:bounds-diagnostic}

The width of the interval $[L_k(r), U_k]$ quantifies confidence in the point estimate. Narrow intervals certify the true geometric mean is pinned down; wide intervals signal many compatible spectra exist.

\begin{center}
\fbox{\parbox{0.95\columnwidth}{
\textbf{Decision Rule} (interval width $U - L$ diagnoses reliability):
\begin{enumerate}[nosep,leftmargin=*]
\item Compute point estimate $\widehat{\GM}$ and bounds $[L, U]$
\item If $\widehat{\GM} \in [L, U]$: use the estimate; otherwise use the nearest bound
\item Report both the (clipped) estimate and the interval $[L, U]$
\end{enumerate}
}}
\end{center}

\begin{remark}[Diagnostic patterns]
If $U_{\mathrm{Rodin}} \ll U_{\mathrm{LS}}$: spectrum likely has outliers.
If $U_{\mathrm{LS}} \ll U_{\mathrm{Rodin}}$: spectrum likely smooth.
The pattern of bound tightness reveals spectral structure without computing eigenvalues.
\end{remark}

This section developed three complementary tools: the last-slope bound exploiting
log-concavity of elementary symmetric polynomials, a measure-theoretic framework
reducing optimization from $O(n)$ to $O(k)$ variables, and a diagnostic protocol
converting interval width into reliability assessment. The $k$-trace bounds
transform the impossibility result of Section~\ref{sec:impossibility} into
actionable guidance: when point estimates fail, guaranteed intervals remain valid.
These guarantees assume exact traces. Section~\ref{sec:noisy-traces} analyzes how
trace noise affects both point estimates and bounds, and Section~\ref{sec:experiments}
validates the predictions empirically.


\section{Sensitivity to Noisy Trace Inputs}
\label{sec:noisy-traces}

The bounds of Section~\ref{sec:bounds} guarantee intervals containing the true
geometric mean regardless of spectrum. But those bounds assume exact trace powers
$p_k = \tr(A^k)$---an idealization. In practice, traces are estimated
stochastically via Hutchinson's method or computed in low-precision arithmetic.

\noindent\textbf{Challenge.} How does trace noise propagate through the $k_{0:m}$ estimators? The Lagrange
weights $w_k = (-1)^{k-1}\binom{m}{k}/k$ grow rapidly with $m$---does this cause
catastrophic noise amplification? If so, the interpolation order that is optimal
under exact traces may be the worst choice under noise. We derive closed-form variance expressions and show that noise amplification grows
as $2^m/m^{5/4}$. For typical Hutchinson noise levels ($\eta \approx 1$--$5\%$),
$m = 4$ balances bias and variance; higher orders are counterproductive.

\subsection{Noise Model}
\label{sec:noise-model}

We model noisy trace observations as multiplicative perturbations:
\begin{equation}
\label{eq:noisy-traces}
\hat{p}_k = p_k(1 + \varepsilon_k), \quad k = 1, \ldots, m,
\end{equation}
where $\{\varepsilon_k\}$ are independent, zero-mean random variables with
variance $\Var(\varepsilon_k) = \eta^2$. The relative noise level $\eta > 0$
is user-controlled. As a heuristic, Hutchinson's estimator with $n_v$ random
vectors yields $\eta \approx \sqrt{2/n_v}$ in well-behaved cases. For
finite-precision arithmetic, $\eta$ reflects the relative rounding error.

\begin{remark}[Choice of multiplicative model]
A multiplicative model fits because trace magnitudes vary widely---$p_k$ grows
as $\lambda_{\max}^k$---so additive noise with fixed variance would overwhelm
small traces or vanish for large ones. The i.i.d.\ assumption simplifies the
analysis; in practice, Hutchinson variance for $\tr(A^k)$ can depend on $k$
and matrix structure. See~\cite{roosta2015} for detailed sample complexity bounds.
\end{remark}

\subsection{Noise Propagation to $K(k)$}
\label{sec:noise-propagation}

Cumulant values are $K(k) = \log(n^{k-1} p_k / p_1^k)$. With noisy traces:
\begin{equation}
\hat{K}(k) = \log\left(\frac{n^{k-1} \hat{p}_k}{\hat{p}_1^k}\right)
= \log\left(\frac{n^{k-1} p_k(1+\varepsilon_k)}{p_1^k(1+\varepsilon_1)^k}\right).
\end{equation}
Separating true values from noise:
\begin{equation}
\hat{K}(k) = K(k) + \log(1+\varepsilon_k) - k\log(1+\varepsilon_1).
\end{equation}

For small noise ($|\varepsilon_k| < 1$ almost surely, or $\eta$ small enough
that $\Pr[\varepsilon_k < -1]$ is negligible), Taylor expansion gives
$\log(1+\varepsilon_k) = \varepsilon_k - \tfrac{1}{2}\varepsilon_k^2 +
O(\varepsilon_k^3)$. Keeping only first-order terms:
\begin{equation}
\label{eq:Khat-linearized}
\hat{K}(k) = K(k) + \varepsilon_k - k\varepsilon_1 + O(\eta^2).
\end{equation}
Thus noise in $\hat{K}(k)$ depends on both $\varepsilon_k$ and $\varepsilon_1$
(with coefficient $k$), reflecting normalization by $p_1^k$.

\subsection{Variance of $\widehat{K'(0)}$}
\label{sec:variance-derivation}

The $k_{0:m}$ estimator (Theorem~\ref{thm:weights}) is
\begin{equation}
\widehat{K'(0)} = \sum_{k=2}^{m} w_k \hat{K}(k), \quad
w_k = L'_k(0) = \frac{(-1)^{k-1}}{k}\binom{m}{k}.
\end{equation}
Substituting \eqref{eq:Khat-linearized}:
\begin{equation}
\widehat{K'(0)} = \sum_{k=2}^{m} w_k K(k) + \sum_{k=2}^{m} w_k \varepsilon_k
- \varepsilon_1 \sum_{k=2}^{m} k\,w_k + O(\eta^2).
\end{equation}
The first sum is the noise-free estimate with bias $b_m = \sum_{k=2}^{m} w_k K(k) - K'(0)$.
The noise-dependent error is
\begin{equation}
\Delta = \sum_{k=2}^{m} w_k \varepsilon_k - \varepsilon_1 \sum_{k=2}^{m} k\,w_k.
\end{equation}

Lagrange differentiation is exact on $\{0,1,\ldots,m\}$, so $\sum_{k=1}^{m} k\,w_k = 1$
(the derivative of $f(x)=x$ at $x=0$). This identity holds because the estimator must exactly recover linear functions. Since $w_1 = m$, we have
$\sum_{k=2}^{m} k\,w_k = 1 - m$ (see Appendix~\ref{app:weight-norm-proof} for details).
Thus:
\begin{equation}
\Delta = \sum_{k=2}^{m} w_k \varepsilon_k + (m-1)\varepsilon_1.
\end{equation}

\begin{theorem}[Variance of $\widehat{K'(0)}$ under noisy traces]
\label{thm:variance}
Under noise model \eqref{eq:noisy-traces} with i.i.d.\ zero-mean $\varepsilon_k$
of variance $\eta^2$:
\begin{equation}
\label{eq:variance-formula}
\Var\bigl(\widehat{K'(0)}\bigr) = \eta^2 \alpha_m^2 + O(\eta^4),
\end{equation}
where the \emph{noise amplification factor} is
\begin{equation}
\label{eq:alpha-m}
\boxed{\alpha_m = \sqrt{\sum_{k=2}^{m} w_k^2 + (m-1)^2}.}
\end{equation}
\end{theorem}

\begin{proof}
Since $\E[\varepsilon_k] = 0$, we have $\E[\Delta] = 0$. By independence:
\begin{equation}
\Var(\Delta) = \sum_{k=2}^{m} w_k^2 \eta^2 + (m-1)^2 \eta^2
= \eta^2 \left[\sum_{k=2}^{m} w_k^2 + (m-1)^2\right].
\end{equation}
The $O(\eta^4)$ term comes from higher-order terms in the Taylor expansion.
\end{proof}

\begin{remark}[Heteroskedastic and correlated noise]
If $\Var(\varepsilon_k) = \eta_k^2$ varies with $k$ or noise is correlated with
covariance matrix $\Sigma$, then $\Var(\Delta) = w^\top \Sigma w$ where
$w = (w_2, \ldots, w_m, m-1)^\top$. The exponential growth of $\lVert w \rVert_2$ dominates,
so the key point remains: high $m$ amplifies noise.
\end{remark}

\subsection{Noise-Induced Bias from Log Nonlinearity}
\label{sec:noise-bias}

First-order analysis shows $\E[\Delta] = 0$, so noise introduces no bias at
order $\eta$. But logarithm nonlinearity creates systematic bias at order $\eta^2$.

Keeping the next Taylor term $\log(1+x) = x - \tfrac{1}{2}x^2 + O(x^3)$:
\begin{align}
\E[\hat{K}(k) - K(k)]
&= \E[\log(1+\varepsilon_k)] - k\,\E[\log(1+\varepsilon_1)] \notag \\
&= -\tfrac{1}{2}\E[\varepsilon_k^2] + \tfrac{k}{2}\E[\varepsilon_1^2] + O(\eta^3)
= \tfrac{k-1}{2}\eta^2 + O(\eta^3).
\end{align}

\begin{proposition}[Noise-induced bias]
\label{prop:noise-bias}
The additional bias due purely to multiplicative noise is
\begin{equation}
\label{eq:noise-bias}
b_{\mathrm{noise}}(m,\eta) = \frac{\eta^2}{2}(1 - H_m) + O(\eta^3),
\end{equation}
where $H_m = \sum_{j=1}^{m} 1/j$ is the $m$-th harmonic number.
\end{proposition}

Since $H_m \approx \ln m + \gamma$ for large $m$, the noise-induced bias
$b_{\mathrm{noise}} \approx -\tfrac{\eta^2}{2}\ln m$ grows only logarithmically.
For typical noise levels ($\eta \lesssim 5\%$) and moderate $m$, this is much
smaller than interpolation bias $b_m$ and variance $\alpha_m\eta$, so it can
usually be neglected.

\subsection{RMSE Decomposition}
\label{sec:rmse}

We measure accuracy via root mean squared error:
\[
\mathrm{RMSE}\bigl(\widehat{K'(0)}\bigr)
  = \sqrt{\E\bigl[(\widehat{K'(0)} - K'(0))^2\bigr]}.
\]

\begin{corollary}[Bias--variance decomposition]
\label{cor:rmse}
The RMSE satisfies
\begin{equation}
\label{eq:rmse-full}
\mathrm{RMSE}^2\bigl(\widehat{K'(0)}\bigr)
\approx \bigl(b_m + b_{\mathrm{noise}}(m,\eta)\bigr)^2 + \alpha_m^2\eta^2.
\end{equation}
Ignoring the smaller $O(\eta^2\log m)$ noise-bias term:
\begin{equation}
\label{eq:rmse}
\boxed{\mathrm{RMSE}\bigl(\widehat{K'(0)}\bigr)
\approx \sqrt{b_m^2 + \alpha_m^2 \eta^2}
\sim \max\{|b_m|, \alpha_m\eta\}}
\end{equation}
(within a factor $\sqrt{2}$).
\end{corollary}

\noindent This decomposition reveals two regimes:
\begin{enumerate}[nosep]
\item \textbf{Bias-dominated} ($\alpha_m \eta \ll |b_m|$):
RMSE $\approx |b_m|$, independent of noise. Increasing $m$ reduces error.

\item \textbf{Noise-dominated} ($\alpha_m \eta \gg |b_m|$):
RMSE $\approx \alpha_m \eta$, growing linearly with $\eta$. Increasing $m$
\emph{increases} error because $\alpha_m$ grows exponentially.
\end{enumerate}
\noindent The \emph{crossover noise level} where bias and variance balance is $\eta_* = |b_m|/\alpha_m$; for $\eta > \eta_*$, reducing $m$ may improve accuracy despite coarser polynomial approximation. Figure~\ref{fig:noisy-crossover} visualizes this trade-off: RMSE plateaus at $|b_m|$ until noise exceeds $\eta_*$, then grows linearly.

\begin{figure}[t]
\centering
\includegraphics[width=0.66\columnwidth]{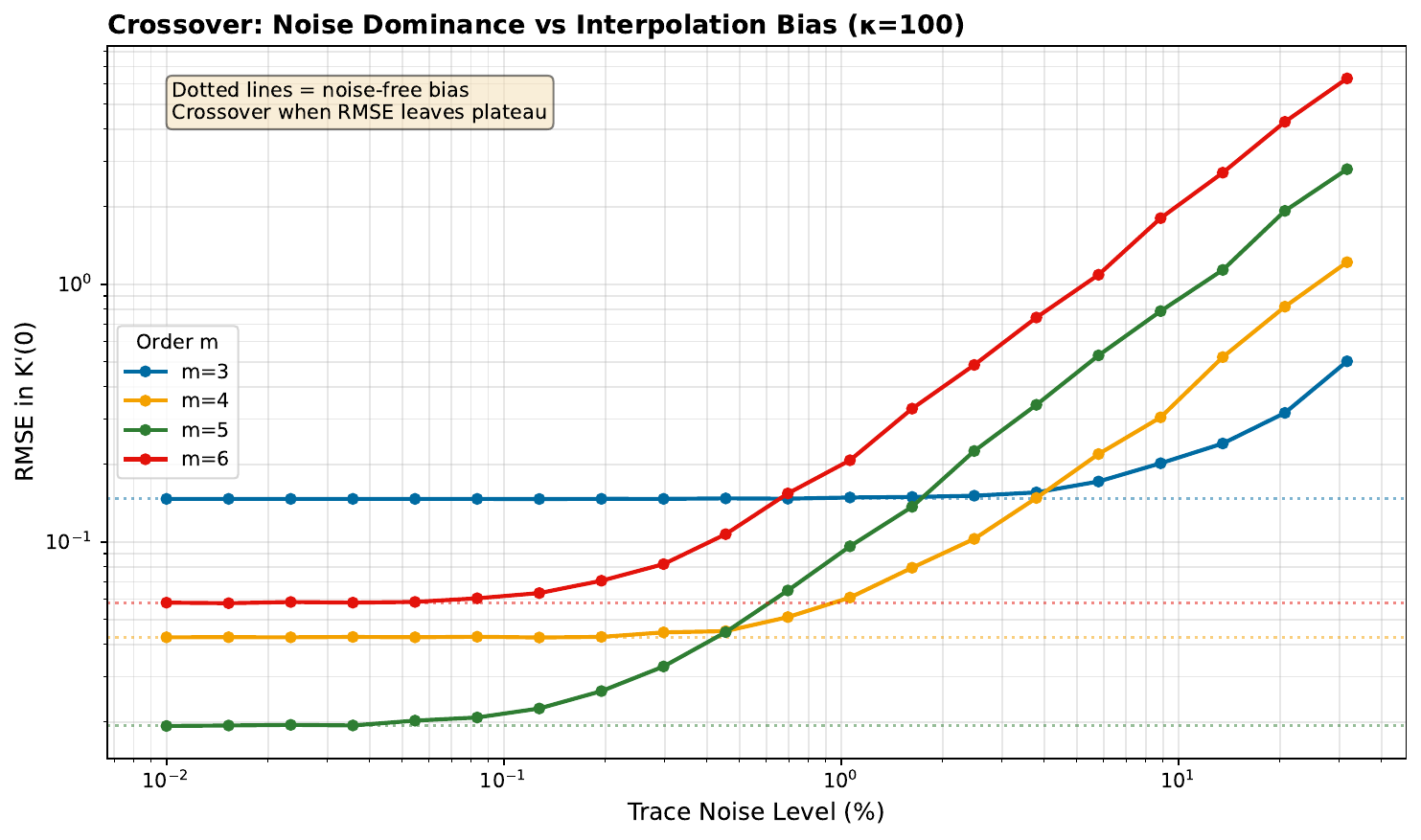}
\caption{Bias-variance trade-off under noisy trace powers.
We add multiplicative noise $\hat{p}_k = p_k(1 + \varepsilon_k)$
with i.i.d.\ $\varepsilon_k \sim \mathcal{N}(0, \eta^2)$ to trace powers
for a geometric spectrum ($n = 1024$, $\kappa = 100$) and report RMSE
in $K'(0)$ over 1000 trials.
Solid curves show $k_{0:m}$ estimators for $m = 3, 4, 5, 6$;
dotted horizontal lines mark noise-free interpolation bias $|b_m|$.
At low noise ($\eta < 0.5\%$), higher $m$ reduces error by decreasing bias.
At high noise ($\eta > 5\%$), higher $m$ increases error due to
exponential noise amplification $\alpha_m \sim 2^m/m^{5/4}$.
The crossover occurs near $\eta_* = |b_m|/\alpha_m$.
For typical Hutchinson noise ($\eta \approx 1$--$5\%$), $m = 4$ provides
the optimal bias-variance trade-off.}
\label{fig:noisy-crossover}
\end{figure}

\subsection{Asymptotic Growth of the Weight Norm}
\label{sec:asymptotic}

The noise amplification factor $\alpha_m$ is dominated by
$\lVert w \rVert_2 = \sqrt{\sum w_k^2}$ for $m \geq 4$ (see Figure~\ref{fig:noisy-traces}c).

\begin{proposition}[$\ell_2$ norm of Lagrange derivative weights]
\label{prop:ell2}
Let $w_k = (-1)^{k-1}\binom{m}{k}/k$ for $k = 2, \ldots, m$. Then
\begin{equation}
\label{eq:sum-w-squared}
\sum_{k=2}^{m} w_k^2 = \frac{4}{\sqrt{\pi}} \cdot \frac{4^m}{m^{5/2}}\bigl(1 + o(1)\bigr)
\end{equation}
as $m \to \infty$, and therefore
\begin{equation}
\label{eq:ell2-asymptotic}
\boxed{\lVert w \rVert_2 = \frac{2}{\pi^{1/4}} \cdot \frac{2^m}{m^{5/4}}\bigl(1 + o(1)\bigr).}
\end{equation}
\end{proposition}

\emph{Proof sketch.}
The sum $\sum w_k^2 = \sum \binom{m}{k}^2/k^2$ is dominated by $k \approx m/2$.
\textbf{Step 1} (Gaussian approximation): $\binom{m}{m/2}^2 \sim 4^m/(\pi m)$
by De Moivre--Laplace.
\textbf{Step 2} (Convert to integral): The effective width is $\sim\sqrt{m}$,
and $1/k^2|_{k=m/2} = 4/m^2$.
\textbf{Step 3} (Evaluate): Combining yields
$\sum w_k^2 \sim (4^m/m) \cdot (4/m^2) \cdot m^{1/2} = 4^m/m^{5/2}$.
See Appendix~\ref{app:weight-norm-proof} for full details and Figure~\ref{fig:asymptotic-validation} for empirical validation.

\subsection{Practical Reference Table}
\label{sec:alpha-table}

\begin{table}[t]
\centering
\caption{Noise amplification factor $\alpha_m$ and predicted standard deviation
at $\eta = 1\%$ noise. For typical Hutchinson noise ($\eta \approx 1$--$5\%$),
$m = 4$ provides a robust bias-variance trade-off.}
\label{tab:alpha}
\begin{tabular}{@{}c|rrrrrrr@{}}
\toprule
$m$ & 2 & 3 & 4 & 5 & 6 & 7 & 8 \\
\midrule
$\lVert w \rVert_2 = \sqrt{\sum w_k^2}$ & 0.50 & 1.54 & 3.29 & 6.14 & 10.78 & 18.49 & 31.61 \\
$(m-1)$ & 1 & 2 & 3 & 4 & 5 & 6 & 7 \\
$\alpha_m$ & 1.12 & 2.52 & 4.45 & 7.33 & 11.88 & 19.44 & 32.38 \\
$\mathrm{Std}(\widehat{K'(0)})$ at $\eta = 1\%$ & 0.011 & 0.025 & 0.045 & 0.073 & 0.119 & 0.194 & 0.324 \\
\bottomrule
\end{tabular}
\end{table}

The weight term $\lVert w \rVert_2$ exceeds the correlation term $(m-1)$ for $m \geq 4$.
For smaller $m$, the $\varepsilon_1$ contribution dominates variance; for
larger $m$, the exponentially growing weights dominate.

\subsection{Practical Guidance}
\label{sec:practical}

The optimal interpolation order balances bias and variance:
\begin{equation}
\label{eq:optimal-m}
m^* = \argmin_m \sqrt{b_m^2 + \alpha_m^2 \eta^2}.
\end{equation}

\begin{itemize}
\item \textbf{High-precision traces} ($\eta < 0.1\%$): Use $m = 5$ or $m = 6$
for minimal bias.

\item \textbf{Hutchinson estimator} ($\eta \approx 1$--$5\%$): Use $m = 4$.
This gives a robust trade-off: bias is $O(10^{-2})$ for moderate $\kappa$,
and $\alpha_4 \approx 4.5$ keeps noise amplification manageable.

\item \textbf{Very noisy traces} ($\eta > 10\%$): Use $m = 3$ or even $m = 2$
to avoid catastrophic noise amplification, accepting larger interpolation bias.
\end{itemize}

\begin{remark}[Adaptive order selection]
In practice, estimate $\eta$ from trace variance and $|b_m|$ from spectrum shape
(if known), then solve \eqref{eq:optimal-m} numerically. Or compute $k_{0:m}$
for $m = 2, 3, 4, 5$ and select the value with minimal apparent variance via bootstrap.
\end{remark}

The theory predicts that $m = 4$ balances bias and variance for typical Hutchinson
noise levels. Section~\ref{sec:experiments} validates this prediction on realistic
spectra and demonstrates that the optimal order shifts toward smaller $m$ as noise
increases, exactly as the $\alpha_m \sim 2^m/m^{5/4}$ growth predicts.

\section{Numerical Experiments}
\label{sec:experiments}

Sections~\ref{sec:estimators}--\ref{sec:noisy-traces} developed the theory; we now test it empirically. Three practical questions remain: How accurate are the $k_{0:m}$ estimators on realistic spectra? Which interpolation order $m$ should practitioners choose? Do the guaranteed bounds detect failure? We answer them with systematic experiments.

\textbf{Experimental questions.} (1) Which interpolation orders $m$ work best for different spectra and condition numbers? (2) Do the bounds detect failures before they cause harm? (3) Does the Taylor radius $R(\kappa)$ predict divergence? We test $k_{0:m}$ for $m = 2, 3, \ldots, 8, 16, 32$ on six spectrum types using exact eigenvalues to isolate interpolation error from stochastic trace variance.

\textbf{Error metric.} All tables and figures report relative error as $(\hat{K}'(0) - K'(0))/|K'(0)| \times 100\%$, where $\hat{K}'(0)$ is the estimate and $K'(0)$ is the true value. Positive errors indicate overestimation; negative errors indicate underestimation.

\textbf{Summary of findings.} For $\kappa \in [50, 200]$ the orders $m = 4, 5, 6$ are typically best on smooth spectra (geometric and uniform), yielding low double-digit percent error in $K'(0)$ or better. The bounds flag every failure, and the Taylor radius $R(\kappa)$ limits usable $m$. The tightest bound depends on the spectrum, so we report multiple upper and lower bounds.

\subsection{Spectrum Types}

We test six spectra grouped by regularity (Table~\ref{tab:spectra}). All have $n$ eigenvalues scaled to condition number $\kappa = \lambda_{\max}/\lambda_{\min}$ with $\lambda_{\min} = 1$. Smooth spectra spread eigenvalues continuously; intermediate spectra introduce moderate irregularity; pathological spectra concentrate mass at extremes to stress the interpolation.

\begin{table}[t]
\centering
\caption{Test spectra grouped by regularity. Smooth spectra spread eigenvalues continuously, intermediate spectra introduce moderate irregularity, and pathological spectra concentrate mass at extremes, stressing the interpolation.}
\label{tab:spectra}
\small
\begin{tabular}{@{}llp{5.5cm}@{}}
\toprule
\multicolumn{3}{@{}l}{\textit{Smooth spectra}} \\
\midrule
Geometric & $\lambda_i = \kappa^{(i-1)/(n-1)}$ & Log-uniform spacing; smoothest case \\[3pt]
Uniform & $\lambda_i = 1 + (\kappa-1)\frac{i-1}{n-1}$ & Linearly spaced \\[3pt]
\midrule
\multicolumn{3}{@{}l}{\textit{Intermediate spectra}} \\
\midrule
Lognormal & $\lambda_i \propto e^{z_i},\; z_i \sim \mathcal{N}(0, \sigma^2)$ & Moderate irregularity; $\sigma = \tfrac{1}{4}\log\kappa$ \\
\midrule
\multicolumn{3}{@{}l}{\textit{Pathological spectra}} \\
\midrule
Two-point & $\lambda = \{1^{(n-1)}, \kappa\}$ & Single extreme outlier \\[3pt]
Bimodal & $\lambda = \{1^{(n/2)}, \kappa^{(n/2)}\}$ & Two masses at extremes \\[3pt]
Clustered & $\lambda \in \{1, \kappa^{1/3}, \kappa^{2/3}, \kappa\} \pm 1\%$ & Four clusters, $n/4$ each \\
\bottomrule
\end{tabular}
\end{table}

\subsection{Effect of Interpolation Order \texorpdfstring{$m$}{m}}

Higher $m$ uses more moment information but risks instability. Table~\ref{tab:k0m-errors} reports relative errors for $k_{0:m}$ on geometric spectra with $n = 1024$. Bold entries mark the best $m$ for each $\kappa$.

\begin{table}[t]
\centering
\caption{Relative errors $(\hat{K}'(0) - K'(0))/|K'(0)| \times 100\%$ of $k_{0:m}$ estimators on geometric spectrum ($n = 1024$). Positive values indicate overestimation.}
\label{tab:k0m-errors}
\small
\begin{tabular}{@{}lrrrrrrrrr@{}}
\toprule
$\kappa$ & $k_{0:2}$ & $k_{0:3}$ & $k_{0:4}$ & $k_{0:5}$ & $k_{0:6}$ & $k_{0:7}$ & $k_{0:8}$ & $k_{0:16}$ & $k_{0:32}$ \\
\midrule
2 & $+2.3$ & $-2.0$ & $-0.5$ & $+0.1$ & $+0.1$ & $+0.1$ & $+0.0$ & $+0.0$ & $\mathbf{+0.0}$ \\
5 & $+11.0$ & $-4.8$ & $-5.6$ & $-3.5$ & $-1.3$ & $+0.2$ & $+1.1$ & $\mathbf{-0.1}$ & $-0.2$ \\
10 & $+19.4$ & $-2.6$ & $-8.3$ & $-8.6$ & $-7.0$ & $-4.9$ & $-2.9$ & $+3.5$ & $\mathbf{-0.5}$ \\
20 & $+27.9$ & $+2.8$ & $-7.0$ & $-10.6$ & $-11.3$ & $-10.7$ & $-9.5$ & $\mathbf{+1.3}$ & $+5.5$ \\
50 & $+37.9$ & $+12.1$ & $\mathbf{-0.7}$ & $-7.5$ & $-11.1$ & $-13.0$ & $-13.8$ & $-8.6$ & $+2.9$ \\
100 & $+44.2$ & $+19.2$ & $+5.6$ & $\mathbf{-2.5}$ & $-7.6$ & $-10.8$ & $-12.8$ & $-14.1$ & $-5.0$ \\
200 & $+49.7$ & $+25.7$ & $+12.0$ & $+3.3$ & $\mathbf{-2.6}$ & $-6.6$ & $-9.6$ & $-16.3$ & $-12.0$ \\
500 & $+55.5$ & $+33.3$ & $+20.0$ & $+11.0$ & $+4.7$ & $\mathbf{+0.1}$ & $-3.5$ & $-15.1$ & $-16.8$ \\
1000 & $+59.2$ & $+38.3$ & $+25.4$ & $+16.5$ & $+10.1$ & $+5.3$ & $\mathbf{+1.5}$ & $-12.4$ & $-17.5$ \\
\bottomrule
\end{tabular}
\end{table}

\begin{widefig}[t]
\centering
\includegraphics[width=\textwidth]{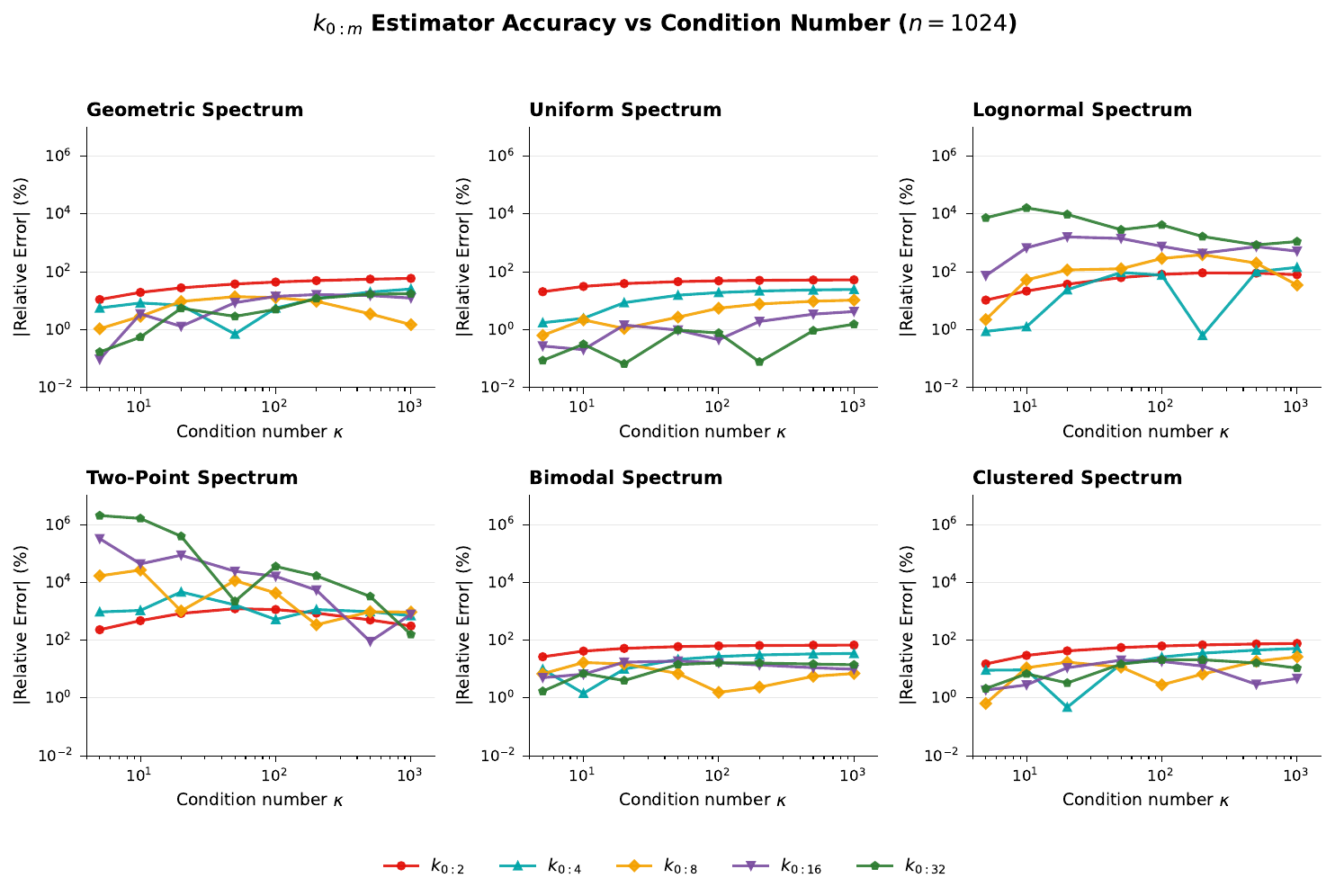}
\caption{Relative error of $k_{0:m}$ estimators across six spectrum types ($n = 1024$). Smooth spectra (geometric, uniform) permit accurate estimation with moderate $m$; lognormal shows larger errors; outlier-dominated spectra (two-point, bimodal, clustered) cause rapid error growth regardless of order.}
\label{fig:k0m-allspectra}
\end{widefig}

Three facts stand out. First, the best $m$ rises with $\kappa$: at $\kappa = 50$ the optimum is $m = 4$, at $\kappa = 100$ it shifts to $m = 5$, and at $\kappa = 500$ the minimum error occurs at $m = 7$. Second, excessively high $m$ degrades accuracy---$k_{0:16}$ and $k_{0:32}$ often produce larger errors than $k_{0:6}$ or $k_{0:8}$, reflecting the Runge phenomenon. Third, the error sign does not follow a simple $(-1)^m$ alternation: $k_{0:3}$ has positive error at $\kappa = 100$ despite $m$ being odd.

\subsection{Optimal Order Selection}

Table~\ref{tab:optimal-m} reports $m^*$ minimizing $|\text{error}|$ for each condition number. The pattern is clear: optimal $m^*$ increases with $\kappa$ up to a stability limit, settling into $m \in \{4, \ldots, 8\}$ for $\kappa \in [50, 1000]$.

\begin{table}[t]
\centering
\caption{Optimal order $m^*$ for geometric spectrum ($n = 1024$)}
\label{tab:optimal-m}
\begin{tabular}{@{}lrrrrrrrrrr@{}}
\toprule
$\kappa$ & 2 & 5 & 10 & 20 & 50 & 100 & 200 & 500 & 1000 & 5000 \\
\midrule
$m^*$ & 32 & 16 & 32 & 16 & 4 & 5 & 6 & 7 & 8 & 16 \\
$|\text{error}|$ at $m^*$ & 0.00 & 0.09 & 0.55 & 1.30 & 0.70 & 2.52 & 2.56 & 0.06 & 1.49 & 3.69 \\
\bottomrule
\end{tabular}
\end{table}

To visualize this trade-off across the $(\kappa, m)$ parameter space, Figure~\ref{fig:optimal-order-heatmap} summarizes optimal interpolation order across spectrum types and condition numbers.

\begin{widefig}[t]
\centering
\includegraphics[width=\textwidth]{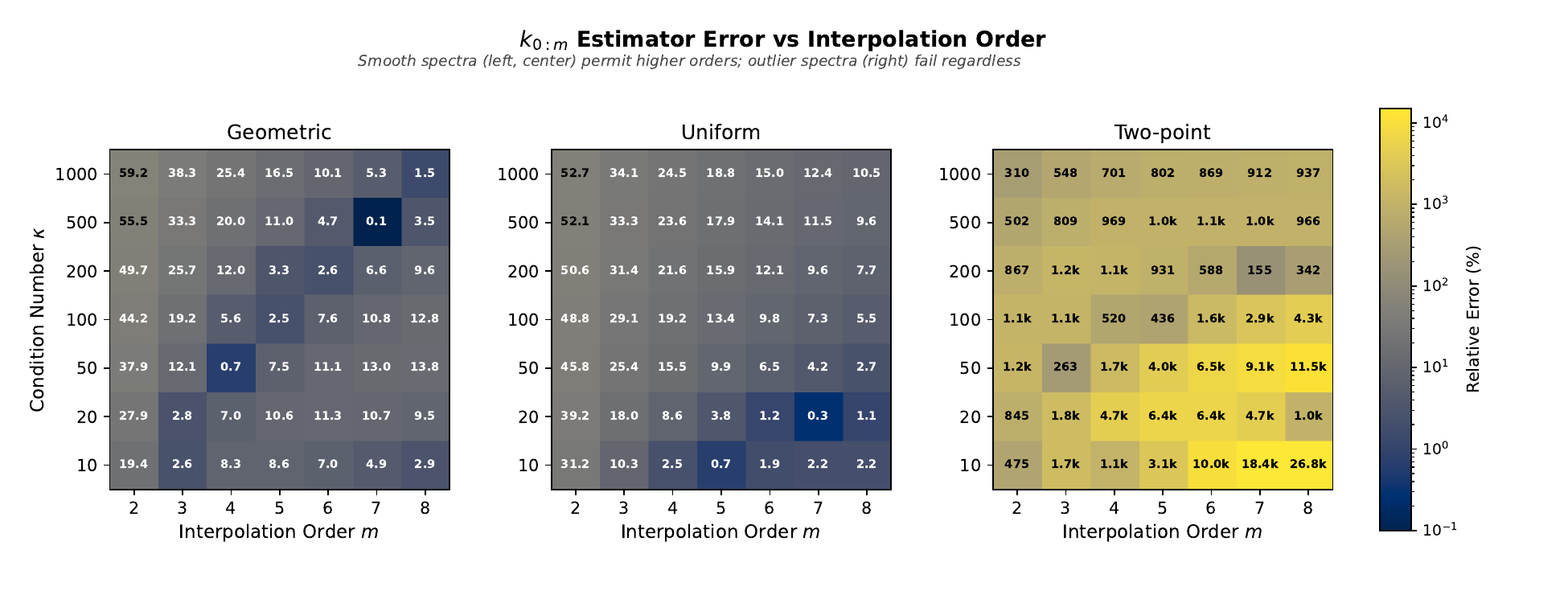}
\caption{Relative error (\%) of the $k_{0:m}$ estimator across condition numbers $\kappa$ and interpolation orders $m$. Left: geometric spectrum. Center: uniform spectrum. Right: two-point spectrum. Each row's minimum indicates the optimal order for that $\kappa$. Geometric and uniform spectra permit higher orders with lower error; two-point spectra fail regardless of order due to Taylor radius limitations (Section~\ref{sec:impossibility}).}
\label{fig:optimal-order-heatmap}
\end{widefig}

\subsection{Estimator Accuracy and Bound Comparison}
\label{sec:bounds-comparison}

Table~\ref{tab:bounds-comparison} presents a unified view of $k_{0:4}$ estimator accuracy and guaranteed bounds across all spectrum types.
We report relative gaps rather than absolute errors to normalize for the scale of the log-determinant.
For each spectrum, we report:
\begin{itemize}[nosep]
\item The $k_{0:4}$ relative error in $K'(0)$: $(\hat{K}'(0) - K'(0))/|K'(0)| \times 100\%$;
\item Upper bound gaps in $K'(0)$: $(\log U - \log(\text{true}))/|\log(\text{true})| \times 100\%$;
\item Lower bound gaps in $K'(0)$: $(\log(\text{true}) - \log L)/|\log(\text{true})| \times 100\%$, with floor $r = \lambda_{\min}/\text{AM}$.
\end{itemize}

\begin{table}[t]
\centering
\caption{Estimator error and bound gaps in $K'(0)$ (\%) by spectrum type ($\kappa=100$, $n=1024$).
Upper gap $= (\log U - \log(\text{true}))/|\log(\text{true})| \times 100$; lower gap $= (\log(\text{true}) - \log L)/|\log(\text{true})| \times 100$.
Bold indicates tightest bound in each category.}
\label{tab:bounds-comparison}
\small
\begin{tabular}{@{}l r | rrrr | rrr @{}}
\toprule
 & & \multicolumn{4}{c|}{Upper bound gap (\%)}  & \multicolumn{3}{c}{Lower bound gap (\%)} \\
\cmidrule(lr){3-6} \cmidrule(lr){7-9}
Spectrum & $k_{0:4}$ & $U_2$ & $U_4$ & $U_8$ & LS & $L_2$ & $L_4$ & $L_8$ \\
\midrule
Geometric & $+5.6$ & 95.6 & \textbf{38.2} & 38.2 & 99.5 & 93.0 & 40.2 & \textbf{10.2} \\
Uniform & $+19.2$ & 94.5 & 28.7 & \textbf{5.6} & 99.7 & 183.4 & 52.4 & \textbf{7.6} \\
Lognormal & $+76.6$ & 86.2 & \textbf{63.9} & 63.9 & 98.2 & 137.4 & \textbf{69.5} & 69.5 \\
Two-point & $-519.8$ & $\approx 0$ & \textbf{$\approx 0$} & $\approx 0$ & 77.2 & 5.1 & \textbf{$\approx 0$} & $\approx 0$ \\
Bimodal & $+55.5$ & 98.3 & \textbf{$\approx 0$} & $\approx 0$ & 99.8 & $\approx 0$ & $\approx 0$ & \textbf{$\approx 0$} \\
Clustered & $+21.3$ & 96.8 & 33.9 & \textbf{6.1} & 99.6 & 57.1 & \textbf{11.9} & $\approx 0$ \\
\bottomrule
\end{tabular}
\end{table}

Several patterns emerge:

\paragraph{Smooth spectra (geometric, uniform).} In $K'(0)$, $k_{0:4}$ gives $+5.6\%$ error on geometric and $+19.2\%$ on uniform at $\kappa=100$. For geometric, the best bounds are $U_4/U_8$ (about $38\%$) and $L_8$ (about $10\%$); for uniform, $U_8$ and $L_8$ tighten to about $6$--$8\%$. The last-slope bound is never tightest.

\paragraph{Pathological spectra (two-point, bimodal).} The estimator fails catastrophically ($-520\%$ to $+56\%$ in $K'(0)$), but the bounds are \emph{exact}: $U_2 = L_2 = \text{true}$ for two-point, $U_4 = L_4 = \text{true}$ for bimodal. This validates that two-point spectra achieve Rodin's tight mean-variance bound~\eqref{eq:rodin}. When estimates fall outside $[L, U]$, the bounds dominate the estimate.

\paragraph{Intermediate cases (lognormal, clustered).} Lognormal is difficult ($+76.6\%$ in $K'(0)$), and its bounds remain wide (about $64$--$70\%$). Clustered spectra are more moderate ($+21.3\%$), and the best bounds tighten to about $6$--$12\%$.

\subsection{Effect of Dimension \texorpdfstring{$n$}{n}}

Interpolation error depends on the spectral distribution shape, not on the number of eigenvalues. Figure~\ref{fig:k0m-nconvergence} shows that for smooth spectra (geometric and uniform), errors stabilize by $n \approx 256$ and remain flat thereafter. Lognormal spectra exhibit erratic behavior for high $m$ due to sampling variance. Pathological spectra (two-point, bimodal) maintain large errors regardless of dimension.

\begin{widefig}[t]
\centering
\includegraphics[width=\textwidth]{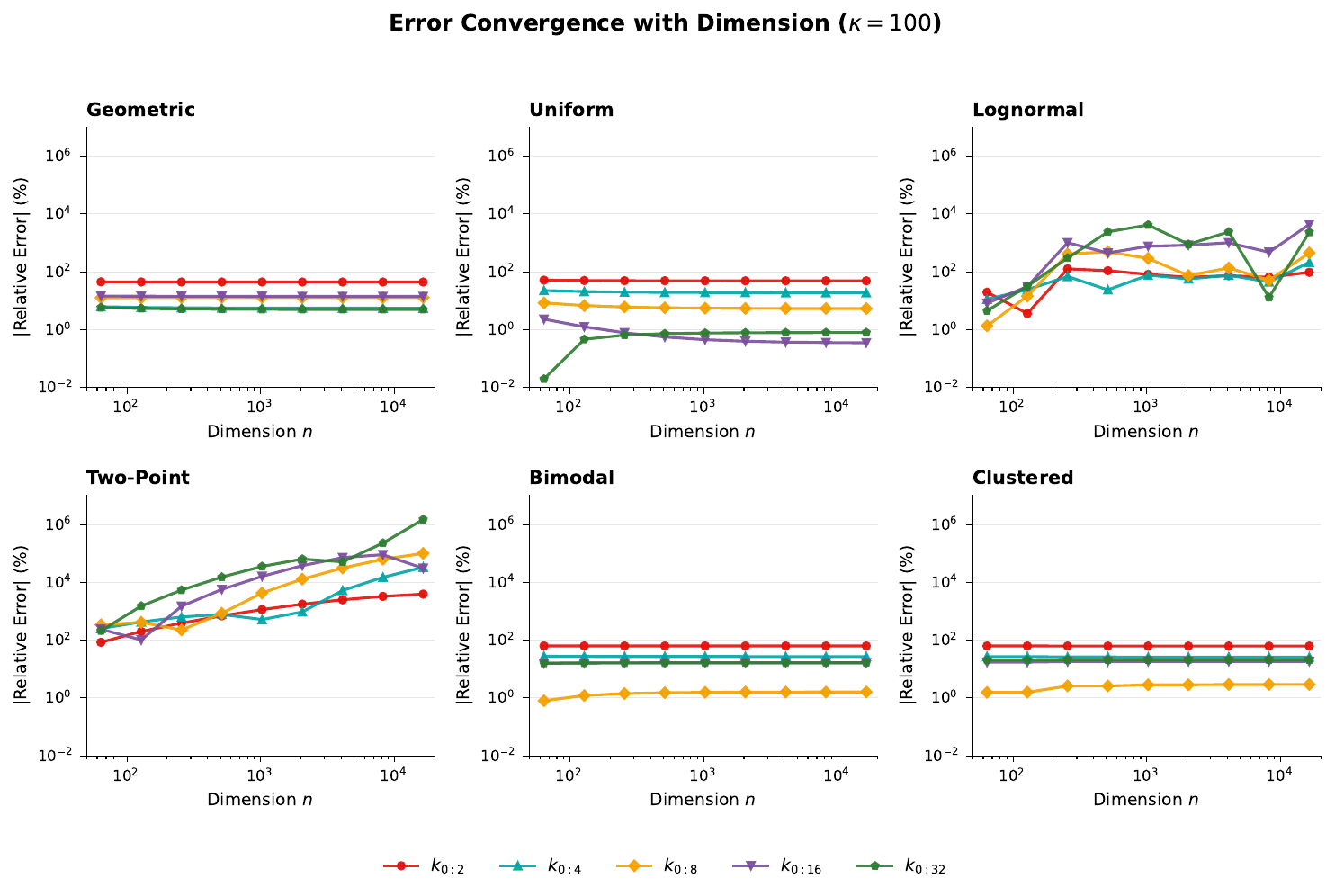}
\caption{Error convergence with dimension $n$ at $\kappa = 100$. Geometric and uniform spectra stabilize rapidly; pathological spectra remain unreliable.}
\label{fig:k0m-nconvergence}
\end{widefig}

\subsection{Taylor Radius and Analyticity}

To isolate fundamental approximation behavior, we evaluate estimators using exact moments---removing all Monte Carlo noise. For normalized eigenvalue $X = \lambda/\mathrm{AM}$, define $M(t) = \mathbb{E}[X^t]$ and $K(t) = \log M(t)$. The Taylor radius $R$ is the distance from $t=0$ to the nearest complex singularity of $K(t)$. For three spectrum families, $R(\kappa)$ has closed form:
\begin{itemize}
\item \textbf{Two-point} (equal weights on $\{1, \kappa\}$): $R = \pi / \log\kappa$
\item \textbf{Log-uniform} on $[1, \kappa]$: $R = 2\pi / \log\kappa$
\item \textbf{Uniform} on $[1, \kappa]$: $R = \sqrt{1 + (2\pi/\log\kappa)^2}$
\end{itemize}

Figure~\ref{fig:taylor-heatmap} shows absolute error $|\hat{K}'(0) - K'(0)|$ over $(\kappa, m)$ with the curve $m = R(\kappa)$ overlaid. Errors remain small when $m < R$ and grow rapidly when $m > R$. The Taylor radius acts as a natural ``budget'' for interpolation order---exceeding it causes divergence.

\begin{widefig}[t]
\centering
\includegraphics[width=\textwidth]{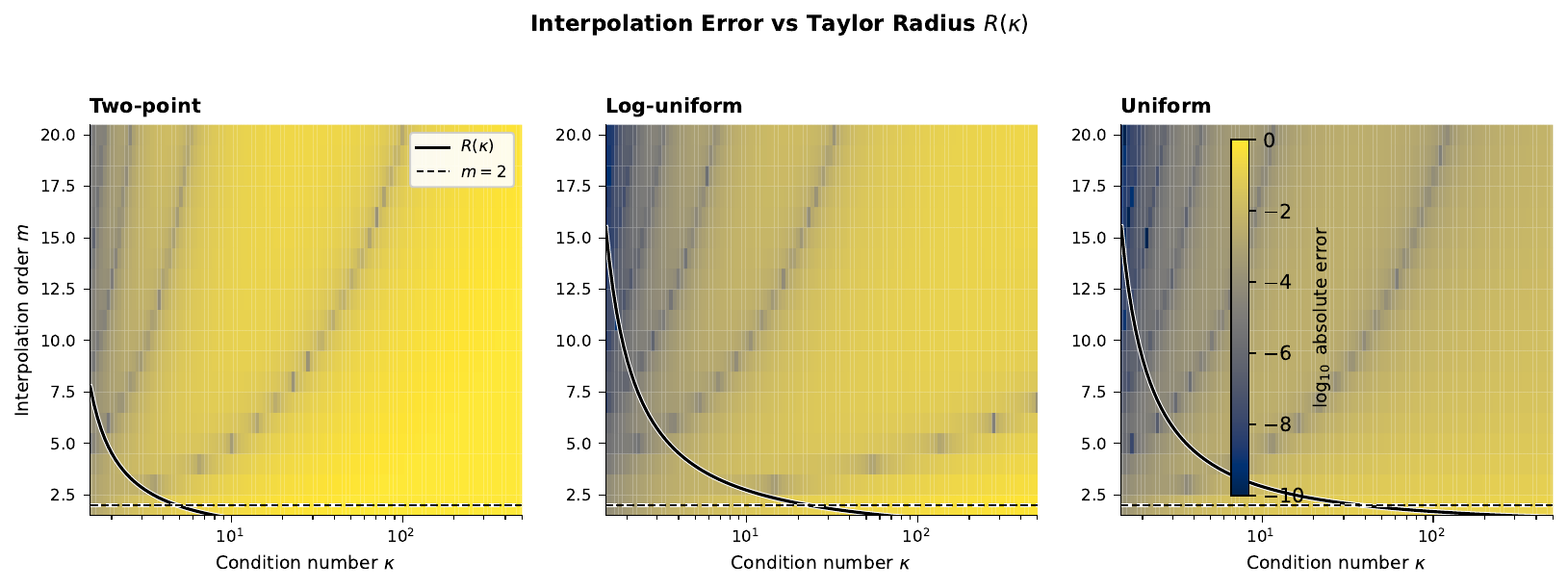}
\caption{Interpolation error vs Taylor radius. Heatmaps show $\log_{10}$ absolute error over $(\kappa, m)$; the solid curve is $R(\kappa)$, the dashed line is $m=2$. Low error (dark) occurs when $m < R$; high error (light) when $m > R$.}
\label{fig:taylor-heatmap}
\end{widefig}

Figure~\ref{fig:taylor-scatter} plots error against normalized order $m/R$. This rescaling collapses behavior across $\kappa$ values: for smooth families, accuracy degrades systematically as $m/R$ increases past 1. The boundary $m = R$ provides a useful rule of thumb.

\begin{widefig}[t]
\centering
\includegraphics[width=\textwidth]{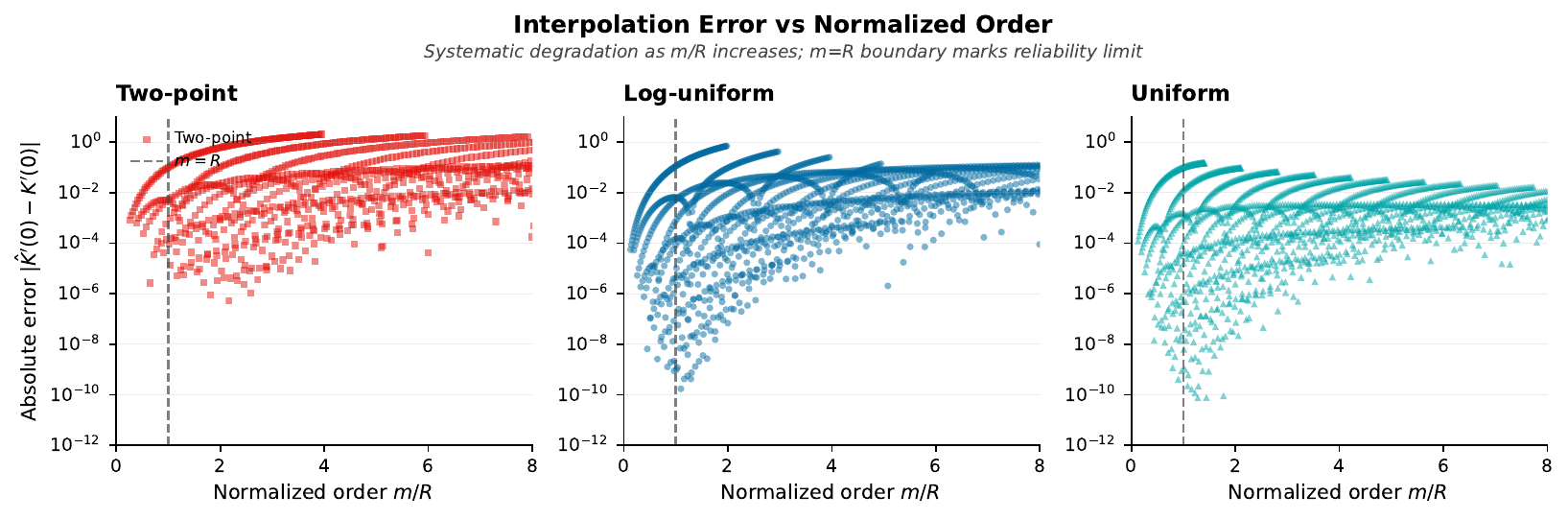}
\caption{Error vs normalized order $m/R$ separated by spectrum type. Each panel shows all $(\kappa, m)$ combinations; the dashed line marks $m = R$. Left: two-point spectra have high baseline error regardless of $m/R$. Center and right: smooth spectra (log-uniform, uniform) show systematic accuracy degradation as $m/R$ increases past 1.}
\label{fig:taylor-scatter}
\end{widefig}

\subsection{Saturation Phenomenon}

Theorem~\ref{thm:impossibility}(b) proved that any continuous estimator must saturate: as $\kappa \to \infty$, the true $K'(0) \to -\infty$ but estimates stay bounded. Figure~\ref{fig:saturation} demonstrates this on two-point spectra. The left panel shows estimates flattening to finite limits (e.g., $k_{0:2} \to -\frac{1}{2}\log 2 \approx -0.35$) while truth diverges. The right panel shows the consequence: relative error grows without bound, eventually exceeding 50\% for all estimators regardless of order $m$.

\begin{widefig}[t]
\centering
\includegraphics[width=\textwidth]{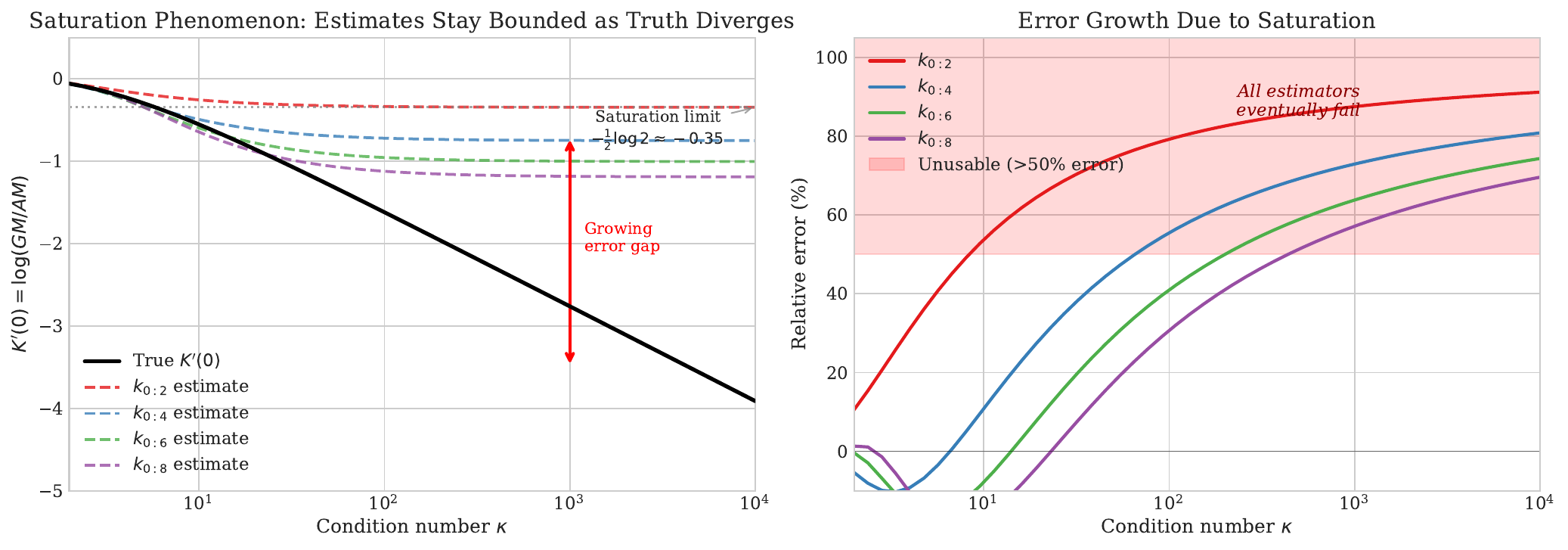}
\caption{Saturation phenomenon on two-point spectra. Left: true $K'(0)$ diverges as $\kappa \to \infty$ while all $k_{0:m}$ estimates saturate at finite limits. Right: relative error grows unboundedly, crossing 50\% (shaded region) at different $\kappa$ thresholds for each $m$.}
\label{fig:saturation}
\end{widefig}

\subsection{Validation of Noise Analysis}
\label{sec:noisy-traces-exp}

Section~\ref{sec:noisy-traces} derives how multiplicative noise in trace powers
propagates through the $k_{0:m}$ estimator, predicting
$\mathrm{RMSE} \approx \max\{|b_m|, \alpha_m \eta\}$
where $\alpha_m \sim 2^m/m^{5/4}$ is the noise amplification factor
(Proposition~\ref{prop:ell2}). We now validate this asymptotic scaling and
demonstrate the bias--variance trade-off experimentally.

\paragraph{Asymptotic validation.}
To validate \eqref{eq:ell2-asymptotic}, we fit empirical $\lVert w \rVert_2$ values to
$c \cdot 2^m / m^a$ for $m \geq 6$:

\begin{table}[t]
\centering
\caption{Validation of asymptotic formula $\lVert w \rVert_2 \sim c \cdot 2^m / m^a$.
Fitted parameters match theoretical predictions.}
\label{tab:asymptotic-validation}
\begin{tabular}{@{}lrrr@{}}
\toprule
Parameter & Theoretical & Fitted & Relative Error \\
\midrule
Exponent $a$ & $5/4 = 1.250$ & $1.301$ & $4.1\%$ \\
Coefficient $c$ & $2/\pi^{1/4} \approx 1.502$ & $1.820$ & $21\%$ \\
Fit quality $R^2$ & --- & $1.0000$ & --- \\
\bottomrule
\end{tabular}
\end{table}

The exponent matches theory within $4\%$. The $21\%$ coefficient discrepancy
comes from discretization error and pre-asymptotic corrections.

\paragraph{Bias--variance trade-off.}
Figure~\ref{fig:noisy-traces} demonstrates the trade-off on a geometric spectrum
with $\kappa = 100$, using 1000 Monte Carlo trials per noise level.
\newpage
\begin{widefig}[t]
\centering
\includegraphics[width=0.8\textwidth]{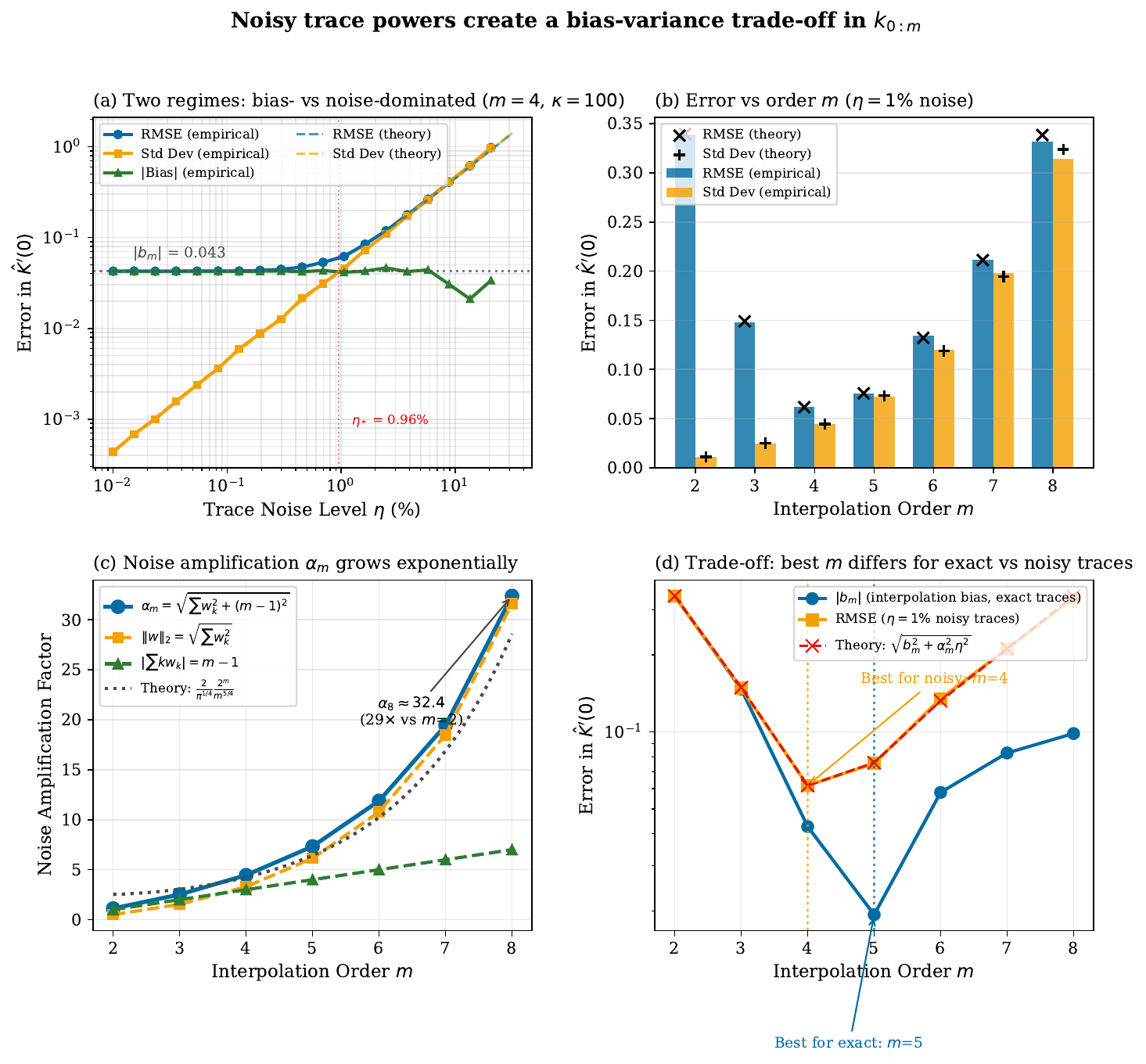}
\caption{\textbf{Noisy trace powers create a bias--variance trade-off in $k_{0:m}$.}
We set $\hat{p}_k = p_k(1 + \varepsilon_k)$ with i.i.d.\
$\varepsilon_k \sim \mathcal{N}(0, \eta^2)$ and record absolute error in $K'(0)$
for a geometric spectrum ($n = 1024$, $\kappa = 100$, 1000 trials).
\textbf{(a)} RMSE, standard deviation, and absolute bias vs.\ $\eta$ for $m=4$;
theory curves (dashed) match empirical values.
\textbf{(b)} RMSE and standard deviation vs.\ $m$ at $\eta = 1\%$; theory markers
($\times$, $+$) overlay empirical bars.
\textbf{(c)} Growth of $\lVert w \rVert_2$ with $m$, showing exponential noise amplification;
theory asymptotic (dotted) matches for $m \geq 4$.
\textbf{(d)} Interpolation bias (exact traces) vs.\ RMSE (noisy traces); optimal
$m$ for bias alone ($m=5$) differs from optimal $m$ for total error ($m=4$).}
\label{fig:noisy-traces}
\end{widefig}

\textbf{Panel (a): Two regimes.}
For fixed $m = 4$, RMSE plateaus at $|b_m|$ when $\eta$ is small, then grows
linearly with slope $\alpha_m$ once $\eta \gtrsim \eta_* \approx 0.96\%$.
Theory curves (dashed) match empirical data closely.

\textbf{Panel (b): Order selection under noise.}
At $\eta = 1\%$, the optimal interpolation order shifts from $m = 5$ (best for
exact traces) to $m = 4$ (best for noisy traces). The exponential growth of
$\alpha_m$ penalizes higher orders more than the polynomial bias decay benefits them.

\textbf{Panel (c): Weight norm growth.}
The $\lVert w \rVert_2$ curve matches the theoretical asymptotic $2^m/m^{5/4}$ for $m \geq 4$.
The $(m-1)$ term is subdominant.

\textbf{Panel (d): Bias vs.\ total error.}
Direct comparison of interpolation bias (exact traces) against RMSE (noisy traces)
confirms the shift in optimal $m$.

\subsection{Summary of Findings}

These experiments yield practical guidelines:

\begin{enumerate}[nosep]
\item \textbf{Order selection.} For $\kappa \in [50, 200]$, use $m \in \{4, 5, 6\}$. Keep $m < R(\kappa)$ to avoid divergence.

\item \textbf{Accuracy expectations.} On geometric spectra, $k_{0:4}$ yields about $6\%$ relative error in $K'(0)$ at $\kappa=100$; uniform and clustered are higher (around $20\%$), and lognormal is higher still. On pathological spectra (two-point, bimodal), estimators fail but bounds are exact.

\item \textbf{Bounds as diagnostics.} The optimization-based bounds $U_k$ and $L_k$ tighten monotonically with $k$. The last-slope bound is never tightest. When estimates fall outside $[L, U]$, use the bound values instead.

\item \textbf{Noisy traces.} For Hutchinson noise ($\eta \approx 1$--$5\%$), use $m = 4$ to balance interpolation bias against noise amplification.
\end{enumerate}

Section~\ref{sec:conclusion} distills these guidelines into deployment recommendations and discusses directions for future work.

\section{Conclusion}
\label{sec:conclusion}

Trace-based log-determinant estimation has three practical outcomes: (1) for smooth spectra with moderate $\kappa$, the $k_{0:m}$ estimators achieve single- to low-double-digit percent error; (2) for pathological spectra, \emph{all} finite-trace methods fail---but bounds $[L, U]$ detect this; (3) the actionable rule is simple: \textbf{clip estimates to $[L, U]$ and always report both}.

\begin{itemize}[leftmargin=*]
\item \textbf{What works.}
The $k_{0:m}$ estimators with $m \in \{4, 5, 6\}$ achieve single- to low-double-digit percent accuracy in $K'(0)$ on geometric and uniform spectra with moderate condition numbers ($\kappa \lesssim 200$). The closed-form weights (Theorem~\ref{thm:weights}) make implementation straightforward in $O(m)$ operations, independent of matrix dimension $n$.

\item \textbf{What does not work.}
The impossibility theorem identifies two fundamental barriers---non-identifiability and saturation---that are information-theoretic, not algorithmic. Two-point and bimodal spectra demonstrate 100\%+ errors that \emph{only the bounds detect}.

\item \textbf{Bounds as guardrails.}
Our bounds provide deterministic intervals $[L, U]$ with $O(k)$ complexity independent of matrix dimension. Estimates inside this interval are typically accurate; clip estimates outside to the nearest bound. For two-point and bimodal spectra, the bounds are exact, while uncorrected estimates err by 100\%+.

\item \textbf{Practical recommendations.}
For moderate condition numbers with high-precision traces, use $k_{0:5}$ or $k_{0:6}$. For stochastic trace estimates (e.g., Hutchinson with $\eta \approx 1\text{--}5\%$ noise), prefer $k_{0:4}$: higher orders suffer from noise amplification $\alpha_m \sim 2^m/m^{5/4}$ (Section~\ref{sec:noisy-traces}). Always compute all three upper bounds and take $U = \min$. The rule is simple: clip estimates outside $[L, U]$ and always report both the estimate and the interval. For matrices with large $\kappa$ or suspected outliers, use stochastic Lanczos quadrature instead.
\end{itemize}


\subsection*{Future Work}

\begin{itemize}[leftmargin=*]
\item \textbf{Stochastic trace estimation.}
This paper uses exact traces to isolate interpolation bias. Combining $k_{0:m}$ with Hutchinson-style trace approximations introduces additional variance. The optimal order balances bias against noise via $m^* = \argmin_{m} \sqrt{b_m^2 + \alpha_m^2\eta^2}$; characterizing this tradeoff precisely remains open.

\item \textbf{Adaptive order selection.}
The optimal interpolation order $m^*$ depends on the unknown condition number $\kappa$. Can the sequence $K(2), K(3), \ldots$ diagnose spectral structure and suggest when to stop?

\item \textbf{Integration with Krylov methods.}
Lanczos iterations produce trace estimates as byproducts. Can our bounds provide online certificates during SLQ iterations? This could enable early stopping when the interval $[L_k(r), U_k]$ is sufficiently tight.

\item \textbf{Lower bounds without spectral floor.}
The lower bound $L_k(r)$ requires $r \leq \lambda_{\min}/\AM$. Can we derive meaningful lower bounds from moments alone, without such a floor?

\item \textbf{Complex Box-Cox transforms (exploratory).}
The Box-Cox family $f_\alpha(z) = (z^\alpha - 1)/\alpha$ admits complex $\alpha = a + ib$.
Appendix~\ref{app:complex-boxcox} shows that oscillatory damping via $\cos(b \ln M_k)$
can dramatically stabilize interpolation for pathological spectra: for two-point spectra,
\emph{oracle-tuned} complex $\alpha$ reduces error from 500\%+ to under 1\%.
However, optimal $(a,b)$ depends on the unknown spectrum and cannot be selected from
traces alone---this is an oracle analysis, not a practical estimator.
Developing a data-driven selection heuristic is an open problem; a coefficient-of-variation
diagnostic shows promise (Appendix~\ref{app:complex-boxcox}).
\end{itemize}

\subsection*{Acknowledgments}

This work was supported by the SPARSITUTE Mathematical Multifaceted Integrated Capability Center (MMICC), funded by the U.S.\ Department of Energy, Office of Science, Advanced Scientific Computing Research (ASCR) program.

The author acknowledges the use of AI assistants (ChatGPT 5.2 OpenAI, Claude Opus 4.5, Anthropic, Gemini 3.0 Google, Kimi-k2 Moonshot AI) for verifying equations, generating figures, converting handwritten notes to \LaTeX\ manuscript form, and expanding proof outlines into complete arguments. All AI-generated content was reviewed and verified by the author.


\printbibliography

\appendix

\section{Appendix}
\label{sec:appendix}

\subsection{Proof of Lagrange Weight Formula}
\label{app:weights}

We prove Theorem~\ref{thm:weights}: for nodes $\{0, 1, \ldots, m\}$,
\[
L'_j(0) = \frac{(-1)^{j-1}}{j}\binom{m}{j}, \quad j = 1, \ldots, m.
\]

\begin{proof}
The Lagrange basis polynomial for node $j$ is:
\[
L_j(t) = \prod_{k=0, k \neq j}^{m} \frac{t - k}{j - k}.
\]

\textbf{Step 1: Denominator.}
\[
D_j := \prod_{k=0, k \neq j}^{m} (j - k) = \prod_{k=0}^{j-1}(j-k) \cdot \prod_{k=j+1}^{m}(j-k).
\]
The first product is $j!$. The second is:
\[
\prod_{k=j+1}^{m}(j-k) = \prod_{\ell=1}^{m-j}(-\ell) = (-1)^{m-j}(m-j)!.
\]
Thus $D_j = j! \cdot (-1)^{m-j}(m-j)!$.

\textbf{Step 2: Express $L_j(t)$ via $\omega(t) = \prod_{k=0}^m (t-k)$.}
\[
L_j(t) = \frac{\omega(t)}{(t-j) D_j}.
\]

\textbf{Step 3: Derivative at zero.}

Since $\omega(0) = 0$ (the product includes $t-0 = t$), we cannot directly apply L'H\^opital. Instead, factor out:
\[
\omega(t) = t \cdot \prod_{k=1}^{m}(t-k) =: t \cdot \tilde{\omega}(t).
\]
Then for $j \geq 1$:
\[
L_j(t) = \frac{t \cdot \tilde{\omega}(t)}{(t-j) D_j}.
\]

At $t = 0$:
\[
L_j(0) = \frac{0 \cdot \tilde{\omega}(0)}{(0-j) D_j} = 0.
\]

For the derivative:
\[
L'_j(t) = \frac{d}{dt}\left[\frac{t \cdot \tilde{\omega}(t)}{(t-j) D_j}\right].
\]

Using quotient rule:
\[
L'_j(t) = \frac{(\tilde{\omega}(t) + t\tilde{\omega}'(t))(t-j) - t\tilde{\omega}(t)}{(t-j)^2 D_j}.
\]

At $t = 0$:
\[
L'_j(0) = \frac{\tilde{\omega}(0) \cdot (-j) - 0}{j^2 D_j} = \frac{-j \cdot \tilde{\omega}(0)}{j^2 D_j} = \frac{-\tilde{\omega}(0)}{j \cdot D_j}.
\]

\textbf{Step 4: Evaluate $\tilde{\omega}(0)$.}
\[
\tilde{\omega}(0) = \prod_{k=1}^{m}(0-k) = \prod_{k=1}^{m}(-k) = (-1)^m m!.
\]

\textbf{Step 5: Combine.}
\[
L'_j(0) = \frac{-(-1)^m m!}{j \cdot j! \cdot (-1)^{m-j}(m-j)!} = \frac{(-1)^{m+1} m!}{j \cdot j! \cdot (-1)^{m-j}(m-j)!}.
\]

Simplify the signs: $(-1)^{m+1} \cdot (-1)^{-(m-j)} = (-1)^{m+1-m+j} = (-1)^{j+1} = (-1)^{j-1} \cdot (-1)^2 = (-1)^{j-1}$.

Thus:
\[
L'_j(0) = \frac{(-1)^{j-1} m!}{j \cdot j! (m-j)!} = \frac{(-1)^{j-1}}{j} \cdot \frac{m!}{j!(m-j)!} = \frac{(-1)^{j-1}}{j}\binom{m}{j}.
\]
\end{proof}

\paragraph{Case $j = 0$: integral evaluation.}
The Lagrange basis polynomials satisfy the partition-of-unity property
$\sum_{j=0}^m L_j(t) = 1$. Differentiating gives $\sum_{j=0}^m L'_j(t) = 0$,
so at $t = 0$:
\[
L'_0(0) = -\sum_{j=1}^{m} L'_j(0) = -\sum_{j=1}^{m} \frac{(-1)^{j-1}}{j}\binom{m}{j}.
\]
We evaluate this sum via the integral identity
\[
\int_0^1 \frac{1 - (1-t)^m}{t}\,dt = H_m.
\]
\emph{Proof of identity:} Substitute $u = 1-t$ to get
$\int_0^1 (1-u^m)/(1-u)\,du = \int_0^1 (1 + u + \cdots + u^{m-1})\,du = H_m$.
Alternatively, expand $(1-t)^m = \sum_{k=0}^m (-1)^k \binom{m}{k} t^k$, so
\[
\frac{1-(1-t)^m}{t} = \sum_{k=1}^{m} (-1)^{k-1}\binom{m}{k} t^{k-1}.
\]
Integrating term-by-term from $0$ to $1$ yields
$\sum_{k=1}^{m} (-1)^{k-1}\binom{m}{k}/k = H_m$.
Therefore $L'_0(0) = -H_m$.

\subsection{Interpolation Error Bound}
\label{app:interp-error}

The standard Lagrange error bound gives
\begin{equation}
K(t) - P_m(t) = \frac{K^{(m+1)}(\xi)}{(m+1)!} \prod_{j=0}^{m}(t - j)
\end{equation}
for some $\xi$ in the convex hull of the nodes. Differentiating at $t = 0$:
\begin{equation}
K'(0) - P'_m(0) = \frac{K^{(m+1)}(\xi)}{(m+1)!} \cdot (-1)^m m! + \text{correction}.
\end{equation}
The sign of the leading term alternates with $m$, but the full error depends on the unknown derivative $K^{(m+1)}(\xi)$. We make no universal sign claim.

\begin{remark}[Noisy trace inputs]
This analysis assumes exact trace powers. When traces are estimated stochastically or computed in finite precision, Section~\ref{sec:noisy-traces} shows that noise amplification grows exponentially with $m$, creating a bias--variance tradeoff that favors smaller interpolation orders.
\end{remark}

\subsection{Alternative Transforms for Moment Interpolation}
\label{app:transform-comparison}

Section~\ref{sec:framework} shows that any analytic transform $f(M(t))$ recovers
$K'(0)$ up to a scale factor. Here we compare three concrete choices:
$M(t) = \E[X^t]$, $K(t) = \log M(t)$, and $L(t) = M(t)^{1/t}$ (so $L(0) = e^{K'(0)}$).
We use the same spectra as Section~\ref{sec:experiments} with $n=1024$ and
$\kappa \in \{5,10,20,50,100,200,500,1000\}$, and summarize the median absolute
relative error in $K'(0)$ across $\kappa$.

\begin{table}[t]
\centering
\small
\setlength{\tabcolsep}{4pt}
\begin{tabular}{@{}lrrrrrr@{}}
\toprule
& \multicolumn{3}{c}{$m=4$} & \multicolumn{3}{c}{$m=6$} \\
\cmidrule(lr){2-4} \cmidrule(lr){5-7}
Spectrum & $K$ & $M$ & $L$ & $K$ & $M$ & $L$ \\
\midrule
Geometric & 7.7 & 105.3 & 36.6 (1) & 7.3 & 714.5 & 29.0 (1) \\
Uniform & 17.3 & 22.2 & 8.9 & 8.1 & 14.5 & 3.4 \\
Lognormal & 9.2 & $5.39\times 10^3$ & 63.5 & 18.1 & $6.66\times 10^5$ & 30.9 \\
Two-point & $1.02\times 10^3$ & $1.04\times 10^7$ & $1.27\times 10^3$ & $1.65\times 10^3$ & $5.11\times 10^{10}$ & $2.80\times 10^3$ \\
Bimodal & 24.4 & 28.9 & 17.1 & 12.3 & 18.7 & 45.2 \\
Clustered & 20.6 & 5.5 & 41.5 (3) & 13.9 & 30.2 & 29.4 (4) \\
\bottomrule
\end{tabular}
\caption{Median absolute relative error (\%) in estimating $K'(0)$ from interpolants
built on $K(t)$, $M(t)$, or $L(t) = M(t)^{1/t}$. Parentheses in the $L$ columns
indicate the number of $\kappa$ values (out of 8) for which the extrapolated $L(0)$
was nonpositive and the error undefined.}
\label{tab:transform-comparison}
\end{table}

\paragraph{Box--Cox sweep.}
We also sweep the Box--Cox family $f_\alpha(z) = (z^\alpha - 1)/\alpha$ with
$\alpha \in [-1,1]$ in steps of $0.1$. For each spectrum, Table~\ref{tab:boxcox-comparison}
reports the best $\alpha$ (minimizing the median absolute relative error in $K'(0)$)
across $\kappa \in \{5,10,20,50,100,200,500,1000\}$, alongside the $\alpha=0$ baseline.

\begin{table}[t]
\centering
\small
\setlength{\tabcolsep}{3pt}
\begin{tabular}{@{}lrrrrrr@{}}
\toprule
& \multicolumn{3}{c}{$m=4$} & \multicolumn{3}{c}{$m=6$} \\
\cmidrule(lr){2-4} \cmidrule(lr){5-7}
Spectrum & best $\alpha$ & best err (\%) & log err (\%) & best $\alpha$ & best err (\%) & log err (\%) \\
\midrule
Geometric & 0.6 & 1.4 & 7.7 & 0.5 & 1.5 & 7.3 \\
Uniform & -0.7 & 10.6 & 17.3 & -0.9 & 4.7 & 8.1 \\
Lognormal & 0.2 & 4.0 & 9.2 & 0.2 & 5.1 & 18.1 \\
Two-point & 0.1 & 759.9 & 1020.2 & -0.1 & 1368.1 & 1652.0 \\
Bimodal & -0.5 & 22.0 & 24.4 & -0.7 & 8.6 & 12.3 \\
Clustered & 1.0 & 5.5 & 20.6 & 0.9 & 3.0 & 13.9 \\
\bottomrule
\end{tabular}
\caption{Box--Cox sweep for $f_\alpha(z) = (z^\alpha - 1)/\alpha$, using
$\alpha \in [-1,1]$ with step $0.1$. Entries report the best $\alpha$ (by
median absolute relative error in $K'(0)$) and compare against $\alpha=0$ (log).}
\label{tab:boxcox-comparison}
\end{table}

\subsection{Complex Box--Cox Transforms}
\label{app:complex-boxcox}

The Box--Cox family extends naturally to complex exponents $\alpha = a + ib$:
\[
f_\alpha(z) = \frac{z^\alpha - 1}{\alpha} = \frac{z^a e^{ib\ln z} - 1}{a + ib}.
\]
For $z > 0$, this is well-defined and $f'_\alpha(1) = 1$ for all $\alpha \neq 0$.
Taking the real part yields a real-valued transform:
\[
g(t) = \Re\!\left[\frac{M(t)^\alpha - 1}{\alpha}\right] = \frac{M(t)^a \cos(b \ln M(t)) - 1}{|a|^2 + b^2}\cdot a + \frac{M(t)^a \sin(b \ln M(t))}{|a|^2 + b^2}\cdot b.
\]

\paragraph{Why complex $\alpha$ helps.}
For spectra where $M(k)$ grows rapidly with $k$ (e.g., two-point), the log-moments
$K(k) = \ln M(k)$ span a large range, causing polynomial interpolation to become
ill-conditioned. The oscillatory factor $\cos(b \ln M(k))$ damps this growth:
instead of $K(k)$ growing from 0 to 11 (for two-point with $\kappa=100$), the transform $g(k)$ oscillates in $[-1, 1]$.

\paragraph{Numerical example.}
For two-point spectrum with $\kappa = 100$, $n = 1024$, $m = 4$:
\begin{center}
\small
\begin{tabular}{@{}lrrrrrl@{}}
\toprule
Transform & $g(0)$ & $g(1)$ & $g(2)$ & $g(3)$ & $g(4)$ & Error \\
\midrule
$\log M$ & 0 & 0 & 2.19 & 6.61 & 11.12 & 520\% \\
$\Re[(M^{1.3i}-1)/(1.3i)]$ & 0 & 0 & 0.22 & 0.57 & 0.73 & \textbf{0.3\%} \\
\bottomrule
\end{tabular}
\end{center}
The oscillatory transform compresses the data range by a factor of 15, enabling
stable polynomial interpolation.

\paragraph{Optimal $b$ selection.}
Table~\ref{tab:complex-boxcox} shows that the optimal imaginary part $b$ varies
with spectrum type and condition number. For smooth spectra (geometric, lognormal),
$b \approx 0$ is near-optimal, so the log transform suffices. For pathological
spectra (two-point), $b \in [1.0, 1.5]$ dramatically reduces error.

\begin{table}[t]
\centering
\small
\begin{tabular}{@{}lrrrrrr@{}}
\toprule
& \multicolumn{3}{c}{$m=4$} & \multicolumn{3}{c}{$m=6$} \\
\cmidrule(lr){2-4} \cmidrule(lr){5-7}
Spectrum & best $(a,b)$ & best err & log err & best $(a,b)$ & best err & log err \\
\midrule
Geometric & $(0.1, 1.1)$ & 0.04\% & 8.3\% & $(0.5, 0.3)$ & 0.02\% & 7.0\% \\
Uniform & $(-0.2, 0)$ & 0.2\% & 2.5\% & $(0.2, 0.5)$ & 0.05\% & 1.9\% \\
Two-point & $(0, 1.3)$ & \textbf{0.3\%} & 520\% & $(-0.1, 0.3)$ & 115\% & 1602\% \\
Bimodal & $(-0.3, 1.1)$ & 0.1\% & 10.8\% & $(0.3, 0.9)$ & 0.3\% & 7.8\% \\
\bottomrule
\end{tabular}
\caption{Complex Box--Cox transforms $f_\alpha$ with $\alpha = a + ib$. Entries show
the best $(a,b)$ pair (searched over $a \in [-0.5, 0.5]$, $b \in [0, 1.5]$) for
$\kappa = 10$, $n = 1024$. Complex $\alpha$ provides dramatic improvement for
two-point spectra.}
\label{tab:complex-boxcox}
\end{table}

\paragraph{Detection heuristic.}
Compute estimates at $\alpha \in \{-0.3, 0, 0.3\}$ (all real). If the coefficient
of variation exceeds 20\%, the log transform is likely unreliable; search over
$b \in [0, 2]$ for a better transform. This heuristic correctly identifies
two-point spectra (CV $\approx 70$--2500\%) while leaving well-behaved spectra
(CV $< 10\%$) unchanged.

\paragraph{Limitations.}
The optimal $(a, b)$ depends on the unknown spectrum; we cannot predict it from
traces alone. This makes the results in Table~\ref{tab:complex-boxcox} an
\emph{oracle analysis} showing theoretical potential, not a practical estimator.
Complex transforms also introduce additional computation (complex arithmetic)
and may complicate error analysis. For typical spectra, the log transform
remains the practical default.

\subsection{Verification of Explicit Formulas}
\label{app:verify}

We verify the $k_{0:4}$ formula by direct computation.

For $m = 4$: $L'_2(0) = -\tfrac{1}{2}\binom{4}{2} = -3$, $L'_3(0) = +\tfrac{1}{3}\binom{4}{3} = +\tfrac{4}{3}$, $L'_4(0) = -\tfrac{1}{4}\binom{4}{4} = -\tfrac{1}{4}$.
Thus:
\[
k_{0:4}: \quad P'(0) = -3 K(2) + \frac{4}{3} K(3) - \frac{1}{4} K(4). \quad \checkmark
\]

\subsection{Newton Identities for Normalized Power Sums}
\label{app:newton}

For normalized eigenvalues $x_i = \lambda_i/\AM$ with $\AM = p_1/n$, define the power sums:
\[
q_k := \sum_{i=1}^n x_i^k.
\]
Expressing $q_k$ in terms of traces:
\[
q_k = \sum_{i=1}^n \frac{\lambda_i^k}{\AM^k} = \frac{p_k}{\AM^k} = \frac{p_k}{(p_1/n)^k} = \frac{n^k p_k}{p_1^k}.
\]
The relationship with the normalized moment $M_k = \frac{1}{n}\sum_i x_i^k$ is $q_k = n \cdot M_k$, giving $M_k = n^{k-1} p_k / p_1^k$. Note that $q_1 = \sum_i x_i = n$ by the normalization convention.

Newton's identities for $(e_1, e_2, \ldots)$ in terms of $(q_1, q_2, \ldots)$: $e_1 = n$, $2e_2 = n^2 - q_2$, $3e_3 = e_2 n - e_1 q_2 + q_3$, and $4e_4 = e_3 n - e_2 q_2 + e_1 q_3 - q_4$.

\subsection{Table of Weights for Reference}
\label{app:weights-table}

\begin{center}
\begin{tabular}{@{}lrrrrrrrr@{}}
\toprule
$m$ & $L'_2$ & $L'_3$ & $L'_4$ & $L'_5$ & $L'_6$ & $L'_7$ & $L'_8$ \\
\midrule
2 & $-1/2$ \\
3 & $-3/2$ & $1/3$ \\
4 & $-3$ & $4/3$ & $-1/4$ \\
5 & $-5$ & $10/3$ & $-5/4$ & $1/5$ \\
6 & $-15/2$ & $20/3$ & $-15/4$ & $6/5$ & $-1/6$ \\
7 & $-21/2$ & $35/3$ & $-35/4$ & $21/5$ & $-7/6$ & $1/7$ \\
8 & $-14$ & $56/3$ & $-35/2$ & $56/5$ & $-14/3$ & $8/7$ & $-1/8$ \\
\bottomrule
\end{tabular}
\end{center}
\noindent General formula: $L'_j(0) = (-1)^{j-1} \binom{m}{j}/j$ for $j = 2, \ldots, m$.

\subsection{Proof of Asymptotic Weight Norm Growth}
\label{app:weight-norm-proof}

We prove Proposition~\ref{prop:ell2}: for $w_k = (-1)^{k-1}\binom{m}{k}/k$,
\[
\lVert w \rVert_2 = \sqrt{\sum_{k=2}^{m} w_k^2} = \frac{2}{\pi^{1/4}} \cdot \frac{2^m}{m^{5/4}}\bigl(1 + o(1)\bigr).
\]

\begin{proof}
We compute $S := \sum_{k=2}^{m} w_k^2 = \sum_{k=2}^{m} \binom{m}{k}^2 / k^2$ using
a saddle-point approximation. The dominant contribution comes from $k \approx m/2$
where the binomial coefficient is maximized.

\paragraph{Step 1: Gaussian approximation of the binomial.}
By the local central limit theorem (De Moivre--Laplace), for $k$ near $m/2$:
\begin{equation}
\binom{m}{k} = \frac{2^m}{\sqrt{\pi m/2}} \exp\left(-\frac{2(k - m/2)^2}{m}\right)
\bigl(1 + O(1/\sqrt{m})\bigr).
\end{equation}
This approximation is valid uniformly for $|k - m/2| = O(\sqrt{m})$.

\paragraph{Step 2: Approximate the sum by an integral.}
Squaring the Gaussian approximation:
\begin{equation}
\binom{m}{k}^2 \approx \frac{4^m}{\pi m/2} \exp\left(-\frac{4(k - m/2)^2}{m}\right).
\end{equation}
Replace the sum by an integral (justified since the summand is smooth on scale
$\sqrt{m}$):
\begin{equation}
S \approx \int_{-\infty}^{\infty} \frac{4^m}{\pi m/2}
\cdot \frac{1}{k^2} \exp\left(-\frac{4(k - m/2)^2}{m}\right) dk.
\end{equation}
The bounds are extended to $\pm\infty$ because the Gaussian tails decay
exponentially fast beyond the sum's natural range.

\paragraph{Step 3: Change of variables.}
Let $t = (k - m/2)/(\sqrt{m}/2)$, so $k = m/2 + t\sqrt{m}/2$ and $dk = (\sqrt{m}/2)\,dt$.
The Gaussian becomes $e^{-t^2}$:
\begin{equation}
S \approx \frac{4^m}{\pi m/2} \cdot \frac{\sqrt{m}}{2}
\int_{-\infty}^{\infty} \frac{e^{-t^2}}{(m/2 + t\sqrt{m}/2)^2} dt.
\end{equation}

\paragraph{Step 4: Expand the denominator.}
For the integral, $|t| = O(1)$ contributes the dominant mass. Writing
\begin{equation}
k^2 = \left(\frac{m}{2}\right)^2 \left(1 + \frac{t}{\sqrt{m}}\right)^2
\approx \frac{m^2}{4}\left(1 + \frac{2t}{\sqrt{m}}\right),
\end{equation}
to leading order $1/k^2 \approx 4/m^2$. Substituting:
\begin{equation}
S \approx \frac{4^m}{\pi m/2} \cdot \frac{\sqrt{m}}{2} \cdot \frac{4}{m^2}
\int_{-\infty}^{\infty} e^{-t^2} dt
= \frac{4^m}{\pi m/2} \cdot \frac{\sqrt{m}}{2} \cdot \frac{4}{m^2} \cdot \sqrt{\pi}.
\end{equation}

\paragraph{Step 5: Simplify.}
\begin{equation}
S = \frac{4^m \cdot \sqrt{m} \cdot 4 \cdot \sqrt{\pi}}{\pi m/2 \cdot 2 \cdot m^2}
= \frac{4 \cdot 4^m}{\sqrt{\pi} \cdot m^{5/2}}.
\end{equation}
Taking the square root: $\lVert w \rVert_2 = \sqrt{S} = \frac{2}{\pi^{1/4}} \cdot \frac{2^m}{m^{5/4}}$.
\end{proof}

\begin{remark}[Interpretation of the exponent]
\label{rem:exponent-interpretation}
The $m^{-5/2}$ exponent in $\sum w_k^2$ arises from three contributions:
\begin{itemize}
\item $m^{-2}$: the $1/k^2$ factor evaluated at $k \sim m/2$;
\item $m^{-1}$: the squared-binomial peak height ($\sim 4^m/(\pi m)$);
\item $m^{+1/2}$: the effective summation width ($\sim \sqrt{m}$).
\end{itemize}
Together: $\sum w_k^2 \sim 4^m \cdot m^{-2} \cdot m^{-1} \cdot m^{1/2} = 4^m / m^{5/2}$.
Taking the square root gives $\lVert w \rVert_2 \sim 2^m / m^{5/4}$.
\end{remark}

\begin{figure}[t]
\centering
\includegraphics[width=\textwidth]{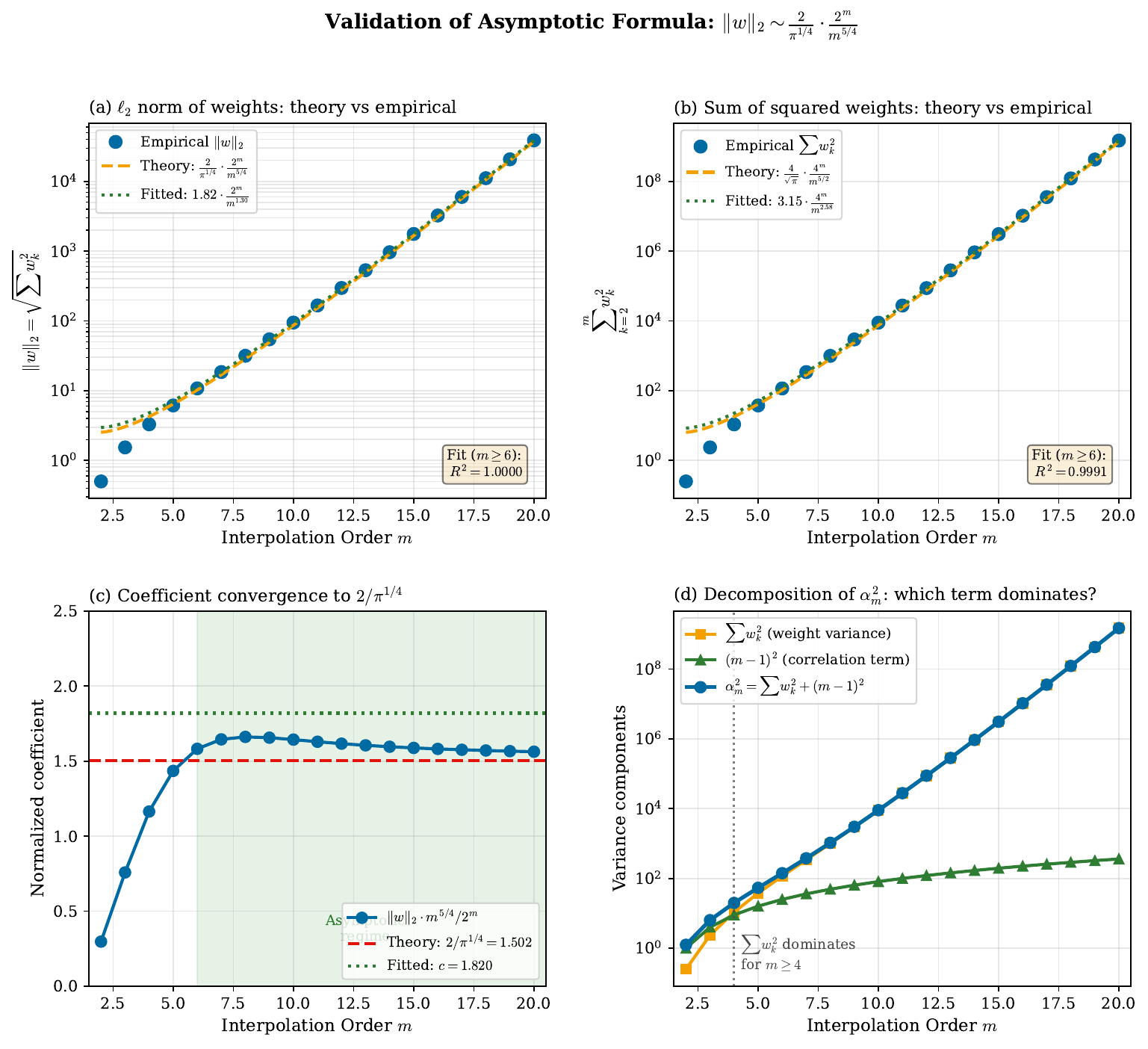}
\caption{Validation of asymptotic weight norm formula (Proposition~\ref{prop:ell2}).
The Lagrange derivative weights $w_k = (-1)^{k-1}\binom{m}{k}/k$ grow
exponentially with interpolation order $m$.
(a)~Weight 2-norm $\lVert w \rVert_2$ vs $m$: empirical values (dots)
match the theoretical curve $2/\pi^{1/4} \cdot 2^m/m^{5/4}$ (dashed).
(b)~Variance component $\sum w_k^2$ matches $4/\sqrt{\pi} \cdot 4^m/m^{5/2}$.
(c)~Normalized coefficient $\lVert w \rVert_2 \cdot m^{5/4}/2^m$ converges
to $2/\pi^{1/4} \approx 1.502$ for $m \geq 6$ (green shaded region).
(d)~Variance decomposition from Theorem~\ref{thm:variance}: the weight term
$\sum w_k^2$ (orange) dominates the correlation term $(m-1)^2$ (green)
for $m \geq 4$, confirming that higher-order estimators suffer
exponential noise amplification.}
\label{fig:asymptotic-validation}
\end{figure}

\subsection{Stochastic Trace Estimation and Polynomial Approximations}
\label{app:trace-poly}

This appendix reviews stochastic trace estimation methods and polynomial/rational approximations of $\log(\cdot)$, which form the foundation of many scalable log-determinant methods.

\paragraph{Stochastic Trace Estimation.}
Hutchinson's estimator~\citep{hutchinson1990} approximates $\tr(B)$ for an implicit matrix $B$ using random vectors $z$ with $\E[zz^\top] = I$:
\begin{equation}
\tr(B) \approx \frac{1}{s}\sum_{i=1}^s z_i^\top B z_i.
\end{equation}
Rademacher vectors ($z_i \in \{-1,+1\}^n$ uniformly) achieve minimum variance among unbiased estimators~\citep{avron2011}. \citet{roosta2015} improved the sample complexity bounds, and \citet{meyer2021} introduced Hutch++, achieving optimal $O(\sqrt{\lVert B \rVert_F/\epsilon})$ sample complexity via low-rank deflation. Probing methods~\citep{bekas2007} estimate matrix diagonals via structured probe vectors, enabling efficient trace computation when the matrix has known sparsity patterns.

For log-determinants, the identity $\log\det(A) = \tr(\log A)$ reduces the problem to estimating $\tr(\log A)$. However, this reduction requires approximating the matrix logarithm---a nontrivial task since $\log(\lambda)$ is unbounded near zero, making accurate approximation difficult for ill-conditioned matrices~\citep{higham2008}.

\paragraph{Polynomial and Rational Approximations.}
Two major approaches approximate $\log(\lambda)$ on $[\lambda_{\min}, \lambda_{\max}]$:

\emph{Chebyshev methods.} \citet{han2015} combine Hutchinson's estimator with Chebyshev polynomial approximation of $\log(x)$, achieving $O(n)$ complexity per sample. The polynomial degree scales as $O(\sqrt{\kappa}\log(1/\epsilon))$ to achieve accuracy $\epsilon$, where $\kappa = \lambda_{\max}/\lambda_{\min}$. This approach is effective for moderate condition numbers but requires many terms when $\kappa$ is large.

\emph{Lanczos quadrature.} \citet{ubaru2017} use Stochastic Lanczos Quadrature (SLQ), interpreting $\tr(f(A))$ as a Riemann-Stieltjes integral and applying Gaussian quadrature rules derived from the Lanczos process. This method adapts to the spectral distribution and often requires fewer iterations than Chebyshev methods for the same accuracy. \citet{saibaba2017} provide rigorous non-asymptotic error bounds for randomized log-determinant estimators based on subspace iteration.

Both approaches require either: (i)~matrix-vector products with $A$ or shifted solves $(A - \sigma I)^{-1}v$, or (ii)~polynomial/rational approximations of $\log(\cdot)$ that converge slowly for large $\kappa$. For background on Krylov methods, see \citet{saad2003}; for polynomial approximation theory, see \citet{trefethen2013}.

\subsection{Computing $k$-Trace Bounds}
\label{app:bounds-computation}

This appendix details the optimization problem for computing the bounds $U_k$ and $L_k(r)$ from Section~\ref{sec:bounds}, explains why $U_3 \approx U_4$ for most spectra, and analyzes computational complexity.

\paragraph{The Optimization Problem.}
Theorem~\ref{thm:compact-support} establishes that $U_k$ is achieved by a discrete measure with at most $k+1$ atoms. This reduces the infinite-dimensional optimization over probability measures to a finite nonlinear program.

\begin{definition}[Upper bound optimization]
\label{def:upper-nlp}
Given normalized moments $M_1 = 1, M_2, \ldots, M_k$:
\begin{equation}
\label{eq:upper-nlp}
U_k = \max_{\substack{w \in \Delta^{k}, \, x \in (0,\infty)^{k+1}}} \;
\sum_{i=1}^{k+1} w_i \log x_i
\quad \text{s.t.} \quad
\sum_{i=1}^{k+1} w_i x_i^j = M_j, \quad j = 1, \ldots, k
\end{equation}
where $\Delta^k = \{w : w_i \geq 0, \sum_i w_i = 1\}$ is the probability simplex.
\end{definition}

The lower bound problem has the same structure but minimizes the objective and constrains $x_i \geq r$:
\begin{equation}
\label{eq:lower-nlp}
L_k(r) = \min_{\substack{w \in \Delta^{k}, \, x \in [r,\infty)^{k+1}}} \;
\sum_{i=1}^{k+1} w_i \log x_i
\quad \text{s.t.} \quad
\sum_{i=1}^{k+1} w_i x_i^j = M_j, \quad j = 1, \ldots, k
\end{equation}

\paragraph{Solution strategy.}
Sequential quadratic programming (SLSQP) with multistart handles the nonconvexity. Initialize with diverse starting points: uniform weights on linearly-spaced atoms, and random draws from Dirichlet-distributed weights with log-normal atoms. Five restarts typically suffice for $k \leq 8$.

\paragraph{Complexity Analysis.}
The key insight is that computational cost depends only on $k$, not on the matrix dimension $n$.

\begin{proposition}[Complexity of $k$-trace bounds]
\label{prop:complexity}
Computing $U_k$ or $L_k(r)$ requires $O(k)$ variables (weights and atoms), $O(k)$ constraints (normalization and $k$ moment equations), and $O(k)$ work per iteration. Total cost is $O(k^2 T)$ with $T < 1000$ iterations. For fixed $k$, complexity is $O(1)$ in $n$.
\end{proposition}

This $O(1)$-in-$n$ complexity arises because moments $M_j$ are precomputed scalars. Once $p_1, \ldots, p_k$ are available (via trace estimation or direct computation), the bound optimization does not reference the matrix again.




\end{document}